\documentclass{article}

\usepackage{microtype}
\usepackage{graphicx}
\usepackage{subfigure}
\usepackage{booktabs} 

\usepackage{hyperref}

\usepackage[accepted]{icml2024}

\usepackage{amsmath}
\usepackage{amssymb}
\usepackage{mathtools}
\usepackage{amsthm}
\usepackage{listings}

\usepackage{thm-restate}

\usepackage{enumitem}
\setlist[itemize]{leftmargin=0.5cm,itemsep=0pt}

\DeclareMathOperator*{\argmax}{arg\,max}

\newcommand{\piref}{\pi^{\mathrm{ref}}}
\newcommand{\KL}{\text{KL}}

\newcommand{\MixtureParam}{\beta}

\usepackage[capitalize,noabbrev]{cleveref}

\theoremstyle{plain}
\newtheorem{theorem}{Theorem}[section]
\newtheorem{proposition}[theorem]{Proposition}
\newtheorem{lemma}[theorem]{Lemma}

\theoremstyle{definition}

\theoremstyle{remark}
\newtheorem{remark}[theorem]{Remark}

\usepackage[textsize=tiny]{todonotes}
\icmltitlerunning{Human Alignment of Large Language Models through Online Preference Optimisation}

\begin{document}

\twocolumn[
\icmltitle{Human Alignment of Large Language Models through \\ Online Preference Optimisation}

\icmlsetsymbol{equal}{*}

\begin{icmlauthorlist}
\icmlauthor{Daniele Calandriello}{equal,dm}
\icmlauthor{Daniel Guo}{equal,dm}
\icmlauthor{R{\'e}mi Munos}{dm}
\icmlauthor{Mark Rowland}{dm}
\icmlauthor{Yunhao Tang}{dm}
\icmlauthor{Bernardo Avila Pires}{dm}
\icmlauthor{Pierre Harvey Richemond}{dm}
\icmlauthor{Charline Le Lan}{dm}
\icmlauthor{Michal Valko}{dm}
\icmlauthor{Tianqi Liu}{dm}
\icmlauthor{Rishabh Joshi}{dm}
\icmlauthor{Zeyu Zheng}{dm}
\icmlauthor{Bilal Piot}{dm}

\end{icmlauthorlist}

\icmlaffiliation{dm}{Google DeepMind}

\icmlcorrespondingauthor{Daniele Calandriello}{dcalandriello@google.com}

\icmlkeywords{Machine Learning, ICML, RLHF, RL, reinforcement learning, policy optimisation}

\vskip 0.3in
]

\printAffiliationsAndNotice{\icmlEqualContribution} 

\begin{abstract}
Ensuring alignment of language models' outputs with human preferences is critical to guarantee a useful, safe, and pleasant user experience. Thus, human alignment has been extensively studied recently and several methods such as Reinforcement Learning from Human Feedback (RLHF), Direct Policy Optimisation (DPO) and Sequence Likelihood Calibration (SLiC) have emerged. In this paper, our contribution is two-fold. First, we show the equivalence between two recent alignment methods, namely Identity Policy Optimisation (IPO) and Nash Mirror Descent (Nash-MD). Second, we introduce a generalisation of IPO, named IPO-MD, that leverages the regularised sampling approach proposed by Nash-MD. 

This equivalence may seem surprising at first sight, since IPO is an offline method whereas Nash-MD is an online method using a preference model. However, this equivalence can be proven when we consider the online version of IPO, that is when both generations are sampled by the online policy and annotated by a trained preference model. Optimising the IPO loss with such a stream of data becomes then equivalent to finding the Nash equilibrium of the preference model through self-play. Building on this equivalence, we introduce the IPO-MD algorithm that generates data with a mixture policy (between the online and reference policy) similarly as the general Nash-MD algorithm. We compare online-IPO and IPO-MD to different online versions of existing losses on preference data such as DPO and SLiC on a summarisation task.   
\end{abstract}

\section{Introduction}
Learning from feedback is a common approach to align the behaviour of artificial agents with human preferences \citep{knox2008tamer,griffith2013policy,christiano2017deep,warnell2018deep}. In recent years, reinforcement learning from human feedback has become a common paradigm for fine-tuning large language models \citep{glaese2022improving,chatgpt}.

The traditional approach to fine-tuning of large language models from human preferences is to learn a reward signal under the Bradley-Terry model \citep{bradley1952rank}, and then perform reinforcement learning (RL) against this learnt reward signal \citep{christiano2017deep}. Recently, \citet{rafailov2023direct} proposed an alternative model-free approach, direct policy optimisation (DPO). DPO is mathematically equivalent to the method above, in the sense that the minimiser of the population loss is identical~\citep{azar2023general}, yet DPO bypasses the learning of the reward signal.
Both approaches, however, rely on the Bradley-Terry model.

Recently, two particular approaches to directly optimise against preference probabilities themselves, rather than a Bradley-Terry-derived reward function have been proposed.
Identity preference optimisation \citep[IPO]{azar2023general} is an algorithm that aims to optimise preference probabilities against a fixed data distribution, and does so with an offline contrastive loss, as with DPO.
By contrast, Nash-MD-PG \citep{munos2023nash} is an algorithm that aims to find a Nash equilibrium with respect to the preference probabilities, via online policy gradient updates against a regularised policy. Both algorithms have appealing properties, though on the face of it are unrelated: one is an offline contrastive algorithm optimising against a fixed policy, the other is an online algorithm aiming to find a Nash equilibrium.

In this work we bridge the gap between IPO and Nash-MD-PG, and use this theoretical bridge to propose a novel class of preference optimisation algorithms. 
Specifically, our principal contributions are: first, we identify several key factors of variation between IPO and Nash-MD-PG, including their use of offline/online data, contrastivity of their losses, and the nature of their equilibria. We use this understanding to identify the strengths of these algorithms, and combine these strengths into new preference optimisation algorithms.
This allows us to propose \emph{Online IPO}, an online variant of IPO. In addition, we establish a theoretical connection between Online IPO and \emph{self-play} in the regularised two-player preference game used in deriving Nash-MD-PG. Second, motivated by Online IPO, we propose a preference optimisation algorithm, aiming to capture the best aspects of both IPO and Nash-MD-PG: \emph{IPO-MD}, a version of IPO that interpolates between offline and online variants by using the lagged data distribution of Nash-MD-PG. Finally, we provide an experimental suite contrasting these algorithms in several applications, which provides detailed comparisons between the proposed methods and several baselines, with notable take-aways for practitioners.

\section{Background}
\label{sec:background}

We begin by introducing the central preference optimisation problem, and relevant prior work.

\subsection{Preference optimisation in bandits}

We consider a (non-contextual) bandit problem (rather than a sequential setting), with finite action space $\mathcal{Y}$.
This simplifies the notation considerably, and the ideas presented are straightforwardly extensible to the contextual/sequential setting where the actions/generations $y$ are conditioned on a state/prompt $x$; we explain this in further detail in the descriptions of implementation details in the experiments.

\textbf{Preferences.}
A \emph{preference function} $p : \mathcal{Y} \times \mathcal{Y} \rightarrow [0,1]$ specifies pairwise preference probabilities between elements of $\mathcal{Y}$. Specifically, given $y, y' \in \mathcal{Y}$, $p(y \succ y')$ is the probability that $y$ is preferred to $y'$. We will assume that preference functions satisfy the symmetry condition $p(y' \succ y) = 1 - p(y \succ y')$.

At a high level, the goal of learning in this context is to find a policy in $\Delta(\mathcal{Y})$ (the set of distributions over $\mathcal{Y}$) that tends to select actions $y \in \mathcal{Y}$ that are preferred over others. This is of course not a precise mathematical objective as stated, and there are several distinct ways in which this can be formalised, which we explore in greater detail below.

\textbf{Data model.}
Actions are sampled from policies $\mu, \mu' \in \Delta(\mathcal{Y})$ as 
$Y \sim \mu, Y' \sim \mu'$ (we write $Y, Y' \sim \mu$ when $\mu = \mu'$).
Given a preference function $p$, the \emph{preference distribution} $\lambda_p$ is defined for $y, y' \in \mathcal{Y}$ as the distribution corresponding to the following sampling procedure:
\[
\lambda_p(y, y') \mbox{ yields } \begin{cases}
(y, y') \mbox{ with probability } p(y \succ y') \\
(y', y) \mbox{ with probability } 1 - p(y \succ y') \, .
\end{cases}
\]

We will frequently make use of a collection of samples drawn from the data-generating policies and the preference distribution, which we denote by $(y^+_i, y^-_i)_{i=1}^N$.
In this case we will say the data is sampled from $(\mu, \mu', \lambda_p)$ (or $(\mu, \lambda_p)$ when $\mu = \mu'$).
In \emph{offline} settings, the data is typically generated from a fixed policy $\mu$, whereas in \emph{online} settings, new data can be generated from a currently estimated policy $\pi$.

Another important aspect of the data available to the learner is an initial policy $\piref \in \Delta(\mathcal{Y})$. Typically, this policy provides acceptable behaviour in some (though not all) aspects of interest in the problem, and may be used to define a trust region for the policy optimisation problem. Our central contributions in this paper draw together, unify, and contrast several different approaches to preference optimisation; we now recall these below.

\subsection{RLHF with a Bradley-Terry reward model}
\label{sec:rlhf}

The approach to preference optimisation proposed by \citet{christiano2017deep} is split into two steps. First, a reward model $r : \mathcal{Y} \rightarrow \mathbb{R}$ is fitted via the Bradley-Terry model \citep{bradley1952rank}. In more detail, the preference probability $p(y \succ y')$ is approximated by the logistic function
\begin{align*}
    \sigma(r(y) - r(y')) =
    \frac{e^{r(y) - r(y')}}{1 + e^{r(y) - r(y')}} \, ,    
\end{align*}
and a reward function $\hat{r}$ is learnt essentially by performing maximum likelihood in this model, given a collection of observed preferences $(y^+_i, y^-_i)_{i=1}^N$ sampled from $(\mu,\lambda_p)$, (approximately) maximising the objective
\begin{align}\label{eq:reward-obj}
    \frac{1}{N}\sum_{i=1}^N \log
        \left(
            \sigma(r(y^+_i) - r(y^-_i))
        \right) \, .
\end{align}
Policy optimisation then proceeds by aiming to maximise the expected reward $\hat{r}$ under the $\pi$, subject to a KL constraint against the initial policy:
\begin{align}\label{eq:rlhf}
    \argmax_{\pi \in \Delta(\mathcal{Y})} \ \bigg[ \mathop{\mathbb{E}}_{Y \sim \pi}[\hat{r}(Y)] - \tau \text{KL}( \pi \; || \; \piref) \bigg] \, ,
\end{align}
with $\tau > 0$ a temperature parameter controlling the degree of regularisation towards $\piref$. Note that the solution $\pi^*$ to Equation~\eqref{eq:rlhf} is available in closed form, as
\begin{align}\label{eq:rlhf-closed-form}
    \pi^*(y) \propto \piref(y) \exp( \tau^{-1} \hat{r}(y)) \, .
\end{align}

\subsection{Direct preference optimisation (DPO)}
\label{sec:dpo}

\citet{rafailov2023direct} propose \emph{direct policy optimisation} (DPO) as an alternative to RLHF as described above, noting that with the closed form in Equation~\eqref{eq:rlhf-closed-form}, the learning of a reward function can be completely bypassed, instead reparametrising the optimal reward in terms of the optimal policy, substituting into Equation~\eqref{eq:reward-obj}, and aiming to maximise the resulting objective with respect to $\pi$:
\begin{align}\label{eq:reward-obj-2}
    \frac{1}{N}\sum_{i=1}^N \log
        \left(
            \sigma
                \left(
                    \tau \; \log\left(
                                \frac{\pi(y^+_i) \piref(y^-_i)}
                                     {\pi(y^-_i) \piref(y^+_i) }
                                \right)
                \right)
        \right) \, .
\end{align}

The derivation of DPO implies that, mathematically, it yields the same optimal policy as the RLHF approach in Section~\ref{sec:rlhf}, which is a regularised optimiser for the reward model $\hat{r}$.

\subsection{Sequence Likelihood Calibration (SLiC)}
\label{sec:slic}

\citet{zhao2023slichf} propose \emph{sequence likelihood calibration} (SLiC) as an alternative to RLHF. Soon after, \citet{liu2023statistical} refine the SLiC loss by normalising the policy probabilities with the reference policy probabilities in order to get a regularised offline loss. In the remaining, we will refer to the following loss, with a dataset $(y^+_i, y^-_i)_{i=1}^N$ sampled from $(\mu, \lambda_p)$:
\begin{align}\label{eq:offline-slic}
   \frac{1}{N}\sum_{i=1}^N
   \max\left(0, 1 - \tau\log \bigg( \frac{\pi(y^+_i)\piref(y^-_i)}{\pi(y^-_i)\piref(y^+_i)} \bigg) \right) \, 
\end{align}
as the SLiC loss. This loss can be interpreted as an hinge-loss variation of DPO~\citep{liu2023statistical}.

\subsection{Identity policy optimisation (IPO)}
\label{sec:offline-ipo}

\citet{azar2023general} note that in general, optimisation of the reward model described in Section~\ref{sec:rlhf}, which is implicitly optimised by DPO, may not always yield an intuitively good policy for the preference probabilities $p$. They also note that in practice, removal of the learnt reward function from the pipeline removes an important source of regularisation in the learning problem, and as such DPO may learn policies that are under-regularised, and converge to deterministic actions.

In order to circumvent these two issues, \citet{azar2023general} propose \emph{identity policy optimisation} (IPO). The derivation begins from the objective of aiming to directly optimise preference probabilities (rather than a proxy reward) against a fixed policy $\mu$:
\begin{align}\label{eq:rpo}
    \mathop{\mathbb{E}}_{\substack{Y \sim \pi \\ Y' \sim \mu}}[p(Y \succ Y')] - \tau \KL(\pi \; || \; \piref) \, .
\end{align}
Similar to Equation~\eqref{eq:rlhf-closed-form}, the optimal policy for this objective is expressible directly as
\begin{align}\label{eq:ipo-opt}
    \pi^*(y) \propto \piref(y) \exp( \tau^{-1} \mathbb{E}_{Y' \sim \mu}[p(y \succ Y')] ) \, ,
\end{align}
which \citet{azar2023general} use to derive the following equivalent offline IPO loss, with a dataset $(y^+_i, y^-_i)_{i=1}^N$ sampled from $(\mu, \lambda_p)$:
\begin{align}\label{eq:offline-ipo}
   \frac{1}{N}\sum_{i=1}^N
    \left( \log \bigg( \frac{\pi(y^+_i)\piref(y^-_i)}{\pi(y^-_i) \piref(y^+_i)} \bigg) -  \tau^{-1}/2 \right)^2 \, .
\end{align}
The quadratic aspect of the loss discourages log-probability ratios between pairs of actions under $\pi$ from deviating too far from those under $\piref$, which ensures regularisation of $\pi$ against $\piref$. By contrast, the minimiser of the DPO empirical loss does not have this property.

\subsection{Nash-MD-PG}
\label{sec:nash-md-pg}

Rather than optimising preference probabilities against an offline dataset generated from some data-generating policy $\mu$, \citet{munos2023nash} propose instead interpreting the objective in Equation~\eqref{eq:rpo} as one player's objective in a two-player, constant-sum game. 
Specifically, two players select policies $\pi_1, \pi_2 \in \Delta(\mathcal{Y})$, with player $i$ receiving payoff
\begin{align}\label{eq:payoff}
    & \mathop{\mathbb{E}}_{\substack{Y \sim \pi_i \\ Y' \sim \pi_{-i}}}[p(Y \succ Y')] \\
    & \qquad- \tau \KL(\pi_i \; || \; \piref) + \tau \KL(\pi_{-i} \; || \; \piref ) \, , \nonumber
\end{align}
where $\pi_{-i}$ denotes the policy of the other player.
Note that holding $\pi_{-i}$ fixed at $\mu$ then yields an equivalent objective to Equation~\eqref{eq:rpo} for $\pi_i$. The proposal of \citet{munos2023nash} is then to find a Nash equilibrium for this game, motivated by the idea that the policies in the resulting Nash equilibrium may be more robust, and are not overly specific to the data-generating distribution $\mu$.
On the flip-side there may be some benefit to regularising the sampling distribution toward the data distribution, and Nash-MD-PG has a parameter $\MixtureParam$ that allows for this tradeoff, with self-play in one extreme $(\MixtureParam=0)$ and sampling from $\mu$ in the other ($\MixtureParam=1$). We denote this algorithm by Nash-MD-PG($\MixtureParam$).

The Nash-MD-PG($\MixtureParam$) algorithm, motivated by mirror-descent approaches to saddle-point computation, aims to do so by updating the policy $\pi$ in the direction of the following policy gradient:
\begin{align*}
    \nabla \log \pi(y) \left(
        p(y \succ y') - \tfrac{1}{2} - \tau \log \left( \frac{\pi(y)}{\piref(y)} \right)
    \right) \, ,
\end{align*}
where importantly $y \sim \pi$, and $y'$ is sampled from a \emph{geometric mixture policy} $\pi^{1-\MixtureParam} (\piref)^{\MixtureParam}$, for some choice of $\MixtureParam \in [0,1]$.

\section{Comparative discussion of preference optimisation algorithms}

The algorithms DPO (\cref{sec:dpo}), SLiC (\cref{sec:slic}), IPO (\cref{sec:offline-ipo}), and Nash-MD-PG (\cref{sec:nash-md-pg}) are distinct along a number of axes.

\textbf{Contrastivity.}
IPO is a \emph{contrastive} algorithm, in that it labels a pair $(Y, Y')$ according to the preference function (via $\lambda_p$) into a positive (preferred) and a negative (not preferred) example $(Y^+, Y^-)$, and then updates the policy via gradients flowing through \emph{both samples}, $\pi(Y^+)$ and $\pi(Y^-)$. 
By contrast, Nash-MD-PG is not contrastive; only the sampled action's policy probability is directly updated based on the preference.
Contrastive algorithms in general have the potential to be more data-efficient, making direct use of both samples in each policy update (see Appendix~\ref{sec:contrastive.vs.non.contrastive} for a condition under which a contrastive gradient estimate has lower variance than its non-contrastive counterpart).

\textbf{Offline/online data.}
IPO is an \emph{offline} algorithm, working with a fixed dataset, while Nash-MD-PG is an \emph{online} algorithm that makes use of sampled actions from both the current estimated policy $\pi$, and a geometric mixture $\pi^{1-\MixtureParam} (\piref)^{\MixtureParam}$. In many settings it may be desirable to work with static, offline data sets, although if it is feasible to gather data online, this can be beneficial from the point of view of effective regularisation.
In the offline setting, the data limits our ability to evaluate the quality of learned policies in terms of preferences, which can lead to learning bad policies that choose actions outside the data distribution, but have low empirical loss.
Similar to offline RL \citep{fujimoto2019off}, this can be mitigated through regularisation to a reference policy $\piref$, possibly related to the sampling distribution $\mu$ \citep{jaques2019way,wu2019behavior}.
An alternative to regularisation is to use online data, and to train $\pi$ on data that is close to what $\pi$ generates.
Online data may not be available in all settings, but, when it is, it can be an effective way to improve performance of learned policies. 

\textbf{Equilibria.}
The IPO loss is a supervised objective. In particular, its data distribution is fixed, and the optimiser is given by the closed-form policy in Equation~\eqref{eq:ipo-opt}, which in particular can be interpreted as a policy that is preferred over the data-generating distribution $\mu$, regularised towards $\piref$. By contrast, Nash-MD-PG is a game-theoretic algorithm, with a loss function whose data distribution and objective change as the estimated policy change themselves. The stationary points for Nash-MD-PG are not defined in a closed-form manner against reference and data policies $\piref$ and $\mu$, but in a self-referential manner.

\textbf{Regularised Sampling.}
As discussed in \cref{sec:nash-md-pg}, Nash-MD-PG allows for sampling from a mixture distribution between $\pi$ and the data-generating distribution $\mu$, and this can also lead to improved performance versus sampling from either policy.

It is clear that there are several combinations of the various properties of previous methods for which no algorithm yet exists, including combinations that could have advantages over previous work, for example an online contrastive method.
\cref{tab:methods} gives an overview of existing methods in terms of what we consider are strengths of these methods, and how the methods introduced in this paper fit in this context, combining these strengths.
\begin{table*}[htb!]
    \centering
    \begin{tabular}{lcccc}
    \toprule
    Method & Contrastive & Online & Regularised Sampling \\
    \midrule
        RLHF \citep{christiano2017deep} &  & \checkmark & \\
        DPO \citep{rafailov2023direct} & \checkmark  &  \\
        SLiC \citep{zhao2023slichf} & \checkmark &  & \\
        IPO \citep{azar2023general} & \checkmark &  & \\
        Self-play (Nash-MD-PG($\MixtureParam=0$)) \citep{munos2023nash} &  & \checkmark &  \\
        Nash-MD-PG \citep{munos2023nash} & & \checkmark & \checkmark  \\
        Online-IPO (\cref{sec:online-ipo}) &  \checkmark & \checkmark & \\
        IPO-MD (\cref{sec:ipo-md}) & \checkmark & \checkmark & \checkmark \\
    \bottomrule
    \end{tabular}
    \caption{Method comparison in terms of their properties.
    }
    \label{tab:methods}
\end{table*}

\section{Online IPO}
\label{sec:online-ipo}

We first aim to bridge the online/offline divide between IPO and Nash-MD-PG, by proposing a new variant of IPO, \emph{Online IPO}, which makes use of an online, shifting data distribution.

\subsection{Algorithm}

To derive an update, we first start with the \emph{population} loss for IPO, which is obtained by taking the minibatch loss in Equation~\eqref{eq:offline-ipo}, and taking an expectation over the dataset (under i.i.d.\ sampling from $(\mu, \lambda_p)$). This yields the (offline) IPO population loss
\begin{align}
    \mathop{\mathbb{E}}_{\substack{Y, Y' \sim \textcolor{red}{\mu} \\ Y^+, Y^- \sim \lambda_p(Y, Y')}}\bigg\lbrack \left( \log \bigg( \frac{\pi(Y^+)\piref(Y^-)}{\pi(Y^-) \piref(Y^+)} \bigg) - \tau^{-1}/2 \right)^2 \bigg\rbrack \, .
\end{align}
The \emph{Online IPO} population loss is given by replacing the static data distribution, highlighted in red above, with the data distribution generated by the current policy, as displayed below
\begin{align}\label{eq:online-ipo}
   \mathop{\mathbb{E}}_{\substack{Y, Y' \sim \textcolor{red}{\texttt{SG}[\pi]} \\ Y^+, Y^- \sim \lambda_p(Y, Y')}}\bigg\lbrack \left( \log \bigg( \frac{\pi(Y^+)\piref(Y^-)}{\pi(Y^-) \piref(Y^+)} \bigg) - \tau^{-1}/2 \right)^2 \bigg\rbrack \, .
\end{align}
Here, $\texttt{SG}[\pi]$ denotes a stop-gradient around $\pi$ in the data distribution, meaning that although we generate data from $\pi$ to construct the loss, we do not differentiate through the data-generation process itself.

The population form of the Online IPO loss will be useful in the analysis that follows. We conclude our description of the approach by noting that the sample-based Online IPO loss coincides with the Offline IPO\footnote{For clarity, we will refer to the original formulation of IPO by \citet{azar2023general} as \emph{Offline IPO} in what follows.} loss in Equation~\eqref{eq:offline-ipo}, with the exception that the samples $(y^+_i, y^-_i)_{i=1}^N$ are drawn from the current policy $\pi$.

\subsection{Analysis}

Before studying the performance of the newly derived Online IPO loss empirically, we pause to consider it from a theoretical perspective. In particular, we aim to understand for which policies this loss is stationary.

By the analysis for (Offline) IPO \citep{azar2023general} summarised in Section~\ref{sec:offline-ipo}, the gradient for the Online IPO loss in Equation~\eqref{eq:online-ipo} is zero iff $\pi$ satisfies
\begin{align}\label{eq:online-fixed-point}
    \pi(y) \propto \piref(y) \exp(\tau^{-1} p(y \succ \pi)) \, ,
\end{align}
where we use the shorthand $p(y \succ \pi) = \mathbb{E}_{Y' \sim \pi}[p(y \succ Y')]$. 
This is a fixed-point condition; note that $\pi$ appears on both sides. This in fact says that $\pi$ is a best-response against itself in the regularised game described in the background for Nash-MD-PG in Section~\ref{sec:nash-md-pg}. Hence, if $\pi$ satisfies Equation~\eqref{eq:online-fixed-point}, it must be the Nash equilibrium for this game; we have shown the following.

\begin{proposition}
\label{prop:online-ipo-nash}
    The minimiser of the online IPO objective is the Nash equilibrium of the regularised game described in Equation~\eqref{eq:payoff}.
\end{proposition}

This is perhaps a surprising conclusion. Offline IPO is not motivated by game-theoretic considerations, yet by moving to an online variant, we have obtained a loss whose stationary point is precisely the Nash equilibrium of the preference game optimised by Nash-MD-PG.
In fact, we can go further, and deduce a direct equivalence of expected updates between Online IPO, and \emph{self-play} in this game.
The proof for the following result is given in Appendix~\ref{sec:proofs}.

\begin{restatable}{proposition}{thmOnlineIPOSelfPlay}\label{thm:OnlineIPOSelfPlay}
    The expected gradient of the Online IPO loss in Equation~\eqref{eq:online-ipo} is identical to the self-play update direction in the game with payoff as in Equation~\ref{eq:payoff}.
\end{restatable}

Here, \emph{self-play} refers to the algorithm in which a policy is updated using gradient ascent on its expected payoff in the game described in Equation~\eqref{eq:payoff}, against another player using the same policy:
\begin{align*}
     \nabla_\pi \Bigg( \mathop{\mathbb{E}}_{Y \sim \pi, Y' \sim \texttt{SG}[\pi]}\bigg[p(Y \succ Y') - \tau \KL(\pi \; || \; \piref ) \bigg] \Bigg) \, .
\end{align*}
Note that in expectation, this corresponds to Nash-MD-PG with $\MixtureParam=0$; however, an important difference is that in Online IPO, updates are \emph{contrastive}, which may result in variance reduction of the gradient estimate.

We have therefore established a close connection between Online IPO and Nash equilibria for the regularised game.

\subsection{Online DPO}
\label{subsec:online-dpo}

As the DPO and SLiC losses are similar to IPO, a natural question is whether online variants of DPO and SLiC are also related to the regularised game given \cref{eq:payoff}. 
We explore this question for online DPO in \cref{sec:online-dpo-supp}.
\Cref{lem:online-dpo-kkt} gives conditions for the stationary point of the regularised game and online IPO to be a stationary point of online DPO.
However, the conditions seem difficult to satisfy: For example, \cref{thm:two-class-nash-is-not-dpo-solution} says that there is no 2-action problem for which the condition is satisfied, except when preferences are uniform ($p(1 \succ 2) = \frac{1}{2}$).
In this sense, apart from the trivial uniform-preference case, online DPO and online IPO are different objectives when $|\mathcal{Y}| = 2$.

As for stationary points of online DPO, we show that, under the Bradley-Terry model assumption, the RLHF solution (\cref{eq:rlhf-closed-form}) is a stationary point of online DPO  (\cref{thm:online-dpo-kkt}), coinciding with that of offline DPO.

\section{IPO-MD}
\label{sec:ipo-md}

Having established a connection between Online IPO and self-play, it is natural to consider whether we can improve on self-play by using regularised policies to generate the data, similar to how Nash-MD optimises preferences against a regularised adversary. 

\subsection{Algorithm}

We consider modulating the data distribution used in the IPO loss using the same geometric mixture ($\pi^{1-\MixtureParam} (\piref)^{\MixtureParam}$) between online and reference policies as in Nash-MD, we arrive at a new family of algorithms that we call IPO-MD that directly corresponds to the family of policies for Nash-MD-PG. This leads to the population loss
\begin{align*}
    \mathop{\mathbb{E}}_{\substack{Y, Y' \sim \textcolor{red}{\texttt{SG}[\pi^{1-\MixtureParam} (\piref)^{\MixtureParam}]} \\ Y^+, Y^- \sim \lambda_p(Y, Y')}}\bigg\lbrack \left( \log \bigg( \frac{\pi(Y^+)\piref(Y^-)}{\pi(Y^-) \piref(Y^+)} \bigg) - \tau^{-1}/2 \right)^2 \bigg\rbrack \, .
\end{align*}
where in the data distribution, we write $\pi^{1-\MixtureParam} (\piref)^{\MixtureParam}$ as shorthand for the policy which is given by normalising the (unnormalised) geometric mixture $\pi^{1-\MixtureParam}(y) (\piref)^{\MixtureParam}(y)$. For $\MixtureParam \in [0,1]$, we obtain an algorithm that interpolates between these two in a certain sense, analogous to Nash-MD-PG \citep{munos2023nash}; we call this algorithm IPO-MD($\MixtureParam$). When $\MixtureParam=0$, we obtain Online IPO, i.e. self-play, and when $\MixtureParam=1$, we get a variation of IPO that tries to improve against a fixed policy $\piref$. An interesting observation is that if we use a slightly different mixture policy that mixes with $\mu$, $\pi^{1-\MixtureParam} (\mu)^{\MixtureParam}$, then this actually interpolates between Online IPO and Offline IPO, where for $\MixtureParam=1$ we get back the Offline IPO objective. In practice it is often difficult to get direct access to $\mu$ so we would not be able to form this geometric mixture.  Some example dynamics in tabular settings are plotted in Appendix~\ref{sec:dynamics}.

\subsection{Analysis}

We now compare Nash-MD-PG with the new algorithm class IPO-MD described above. First, these two classes of algorithms are practically different. Nash-MD-PG is on-policy, in that the only gradient contributions appearing in its update are those corresponding to actions sampled under the current policy. By contrast, IPO-MD is an off-policy algorithm, since it updates its current policy based on actions sampled under $\pi^{1-\MixtureParam}(\piref)^{\MixtureParam}$. 

By the analysis of Offline IPO \citep{azar2023general}, we have that any fixed point $\pi^*_\MixtureParam$ of IPO-MD($\MixtureParam$) must satisfy
\begin{align}\label{eq:ipo-md-fp}
    \pi^*_\MixtureParam(y)\! \propto \! \piref(y) \exp\big( \tau^{-1} p(y \succ (\pi^*_\MixtureParam)^{1-\MixtureParam} (\piref)^{\MixtureParam})\big) \, .
\end{align}
In other words, $\pi^*_\MixtureParam$ is a best-response against $(\pi^*_\MixtureParam)^{1-\MixtureParam} (\piref)^{\MixtureParam}$ in the regularised game described in Equation~\eqref{eq:payoff}. But from the description of the Nash-MD-PG($\MixtureParam$) algorithm by \citet{munos2023nash}, we have that a policy $\pi^*$ is stationary under this algorithm's update iff
\begin{align*}
    \pi^*(y) \propto \piref(y) \exp\big( \tau^{-1} p(y \succ (\pi^*)^{1-\MixtureParam} (\piref)^{\MixtureParam})\big) \, .
\end{align*}
Hence, the fixed point of IPO-MD$(\MixtureParam)$ coincides with the fixed point of Nash-MD-PG($\MixtureParam$).

Having established the equivalence of the stationary points of IPO-MD($\MixtureParam$) and Nash-MD-PG($\MixtureParam$), we now study their gradients more generally in the following result; see Appendix~\ref{sec:gradients} for the proof.

\begin{restatable}{proposition}{propIPOGradients}\label{prop:gradients}
    The gradients of the algorithms Nash-MD-PG($\MixtureParam$) and IPO-MD($\MixtureParam$) are, respectively, 
    \begin{eqnarray*}
        g_{\mbox{{\tiny Nash-MD-PG}}(\MixtureParam)} &=& - {\mathbb E}_{y\sim\pi}\left[g(y)\right]\\
        g_{\mbox{{\tiny IPO-MD}}(\MixtureParam)} &=& - \frac{2}{\tau} {\mathbb E}_{y\sim(\pi)^{1-\MixtureParam}(\piref)^\MixtureParam}\left[g(y)\right]
    \end{eqnarray*}
    where $g(y)$ is given by
    \begin{align*}
        \nabla\log\pi(y)\left( p(y\succ (\pi)^{1-\MixtureParam}(\piref)^\MixtureParam)-\tau\log\frac{\pi(y)}{\piref(y)}\right) \, .
    \end{align*}
\end{restatable}

We recover the result mentioned earlier that when $\MixtureParam=0$, IPO-MD($\MixtureParam=0$) (i.e., Online IPO) has a gradient aligned with that of Nash-MD-PG($\MixtureParam=0$) (i.e., Self-Play). Now as soon as $\MixtureParam>0$, the gradients of these two algorithms are different. Interestingly however, as noted above, their fixed points remain the same.

We can also relate the fixed points themselves back to the original regularised game given in Equation~\eqref{eq:payoff}, as described below.

\begin{restatable}{proposition}{propIPOMD}
    Let $\pi^*_\MixtureParam$ be the fixed-point of IPO-MD($\MixtureParam$), satisfying Equation~\eqref{eq:ipo-md-fp}. The policy $\pi'_\MixtureParam=(\pi^*_\MixtureParam)^{1-\MixtureParam}(\piref)^\MixtureParam$ is the Nash equilibrium for the version of the game in Equation~\eqref{eq:payoff} with regularisation parameter $\tau$ modified to $\tau(1-\MixtureParam)^{-1}$.
\end{restatable}

\begin{proof}
    Using the property in Equation~\eqref{eq:ipo-md-fp}, we have
    \begin{align*}
       &  \pi'_\MixtureParam(y)\\
        \propto& (\piref(y))^{1-\MixtureParam} \exp(\tau^{-1}(1-\MixtureParam) p(y\succ \pi'_\MixtureParam)) (\piref(y))^{\MixtureParam}\\
        \propto& \piref(y)  \exp(\tau^{-1}(1-\MixtureParam) p(y\succ \pi'_\MixtureParam)) \, ,
    \end{align*}
    which is the Nash equilibrium condition for the game described, as required.
\end{proof}

\section{Experiments}
\label{sec:experiments}
We present our results on fine-tuning large language models where we compare our algorithms, online-IPO and IPO-MD, against recent baselines. In this section we only present the results for the online versions of IPO, DPO and SLIC to make the comparison against IPO-MD and Nash-MD-PG fair, and drop the corresponding "online-" prefix for simplicity. We refer the reader to the appendix for results concerning the offline versions of those algorithms. Note that to aid interpretability and reproducibility, we now consider the contextual bandit case where the actions $y$, also referred as generations, are conditioned  on a prompt $x$.

\textbf{Setup and algorithms.}
We perform RLHF-style experiments where we initialise from a supervised-fine-tuned checkpoint, and then further fine-tune using one of the following algorithms: RL (regularised policy gradient), IPO, DPO , SLiC, Nash-MD and IPO-MD. These algorithms use either a learned reward model $r_\phi$ (RLHF) or a learned preference model $p_\phi$ (IPO, DPO, SLiC, Nash-MD, and IPO-MD). For our RL baseline, similarly to~\citet{munos2023nash}, we use a regularised policy gradient update:
\begin{align*}
\label{eq:reg-pg}
\mathop{\mathbb{E}}_{\substack{x\sim\rho \\ y\sim\pi_\theta(\cdot|x)}}\!\!\!\!\!\!\big\lbrack \nabla_\theta \log\pi_\theta(y|x) \!\left( r_\phi(y|x) \!-\! \tau \KL(\pi_\theta(\cdot|x), \piref\!(\cdot|x))\right) \!\big\rbrack,
\end{align*}
where $r_\phi(y|x)$ is the reward model's value for generation $y$ and context $x$.

\textbf{Implementation details.}
The contrastive offline algorithms such as IPO, DPO and SliC  directly optimise the policy $\pi_\theta$ by minimising their respective losses over a pairwise dataset $(x_i, y_i^+, y_i^-)_{i=1}^N$ . In practice, we sample batches $(x_i, y_i^+, y_i^-)_{i=1}^B$ of size $B\ll N$ and we minimise the following loss:
\begin{align*}
\frac{1}{B}\sum_{i=1}^B\mathcal{L}_{\texttt{ALGO}}(\theta, x_i, y_i,y'_i),
\end{align*}
where:
\begin{align*}
\mathcal{L}_{\texttt{IPO}}(\theta, x, y, y') \! &= \! \left(\log \left( \frac{\pi_\theta(y|x)\piref(y'|x)}{\pi_\theta(y'|x) \piref(y|x)} \right) - \tau^{-1}/2 \right)^2, 
\\
\mathcal{L}_{\texttt{DPO}}(\theta, x, y, y')\! &= \!\sigma\left( \tau\log\left(\frac{\pi_\theta(y|x) \piref(y'|x)}{\pi_\theta(y'|x) \piref(y|x) }\right)\right),
\\
\mathcal{L}_{\texttt{SLiC}}(\theta, x, y, y')\! &= \! \max\left(0, 1 \!-\! \tau\log \left( \frac{\pi_\theta(y|x)\piref(y'|x)}{\pi_\theta(y'|x)\piref(y|x)} \right) \right) \! .
\end{align*}
One important detail concerning IPO is that we use a simplified loss in our code. One can remark by expanding the square and removing terms that do not depend on $\theta$ that the IPO loss is equivalent to:
\begin{align*}
 -\log \left(\frac{\pi_\theta(y|x)}{\pi_\theta(y'|x)} \right) + \tau \left(\log \left( \frac{\pi_\theta(y|x)\piref(y'|x)}{\pi_\theta(y'|x) \piref(y|x)} \right)\!\right)^2.  
\end{align*}

The contrastive online algorithms such as IPO, DPO and SLiC  use a trained preference model $p_\phi$. To train $p_\phi$, we use a pairwise dataset $(x_i, y_i^+, y_i^-)_{i=1}^N$ and follow the same protocol as~\citet{munos2023nash}. Then, to train the policy $\pi_\theta$, for each context $x_i$ of a batch, we sample two completely new generations $(y_i, y_i')\sim\pi_\theta$ according to $\pi_\theta$ and compute the preference $p_i=p_\phi(y_i\succ y_i'|x_i)$ via the preference model. Then, for each algorithm, we minimise the respective following loss:
\begin{align*}
\frac{1}{B}\sum_{i=1}^B\left( p_i\mathcal{L}_{\texttt{ALGO}}(\theta, x_i, y_i,y'_i) + (1-p_i)\mathcal{L}_{\texttt{ALGO}}(\theta, x_i, y'_i,y_i)\right).
\end{align*}

Finally, regarding IPO-MD, the only difference with (online)-IPO is how the generations are sampled. In theory, we should sample from the mixture $\pi_\theta^{1-\MixtureParam} (\piref)^{\MixtureParam}$ which is not feasible (see~\citep{munos2023nash}). In practice, we sample from the one-step-at-a-time mixture $\hat{\pi}_{\MixtureParam}$, which consists at step $n$ to compute the mixture of logits between the online and reference logits:
\begin{align*}
&\log(\hat{\pi}_{\MixtureParam}(.|y_{0:n-1}, x)) = (1-\MixtureParam) \log(\pi_{\theta}(.|y_{0:n-1}, x)) 
\\
&+ \MixtureParam \log(\piref(.|y_{0:n-1}, x)) + C(y_{0:n-1}, x), 
\end{align*}
where $C(y_{0:n-1}, x)$ is a path-dependent constant and sample according to the softmax of this mixture of logits. This sampling process is not equivalent to sampling from  $\pi_\theta^{1-\MixtureParam} (\piref)^{\MixtureParam}$ as shown in~\citep{munos2023nash}.

\begin{table*}[t]
\centering
{
\setlength\tabcolsep{4.7pt}
\begin{tabular}{c|l|l|l|l|l|l|l|l|l|l|l|l|l|l}
$p(y \succ y')$  & IPO     & IPO-MD        & DPO           & Nash-MD-PG    & SLiC          & RL          \\
\midrule
\textbf{IPO}        & 0.500   & \textbf{0.515} \small{(0.024)} & \textbf{0.608} \small{(0.038)} & \textbf{0.621} \small{(0.030)} & \textbf{0.608} \small{(0.025)} & \textbf{0.791} \small{(0.012)} \\
\textbf{IPO-MD}     & \textbf{0.485} \small{(0.024)} & 0.500    & \textbf{0.600} \small{(0.028)} & \textbf{0.608} \small{(0.026)} & \textbf{0.594} \small{(0.020)} & \textbf{0.778} \small{(0.004)} \\
DPO        & 0.392 \small{(0.038)} & 0.400 \small{(0.028)} & 0.500    & 0.520 \small{(0.041)} & 0.493 \small{(0.040)} & 0.727 \small{(0.020)} \\
Nash-MD-PG & 0.379 \small{(0.030)} & 0.392 \small{(0.026)} & 0.480 \small{(0.041)} & 0.500    & 0.479 \small{(0.029)} & 0.729 \small{(0.020)} \\
SLiC       & 0.392 \small{(0.025)} & 0.406 \small{(0.020)} & 0.507 \small{(0.040)} & 0.521 \small{(0.029)} & 0.500    & 0.728 \small{(0.010)} \\
RL         & 0.209 \small{(0.012)} & 0.222 \small{(0.004)} & 0.273 \small{(0.020)} & 0.271 \small{(0.020)} & 0.272 \small{(0.010)} & 0.500     
\end{tabular}
}
\caption{Side-by-side evaluation for summarisation. Each entry is the average preference of responses generated by the row method ($y$) over the responses generated by the column method ($y'$). We also evaluated each method using 3 different seeds, computed a $3\times3$ comparisons matrix across seeds and report the standard deviation of this matrix's entries.
\vspace{-1\baselineskip}
}
\label{results:pref_sum}
\end{table*}

\textbf{Evaluation tasks and Models.}
In our experiments, we test all of the algorithms on an article summarisation task. We use the dataset described by~\citet{stiennon2020learning} that has been built from the TL;DR dataset~\citep{TLDR}. This is a dataset with pairwise preferences between alternate summaries. We train our preference and reward model  on the train set $D_{\texttt{Train}}$, which contains $92820$ examples. We evaluate reward and preference models on a test set of high confidence data $D_{\texttt{Test}}$ and use the checkpoints with the highest evaluation agreement score. To train the policies with online algorithms, we use prompts of the train set of the XSum dataset~\citep{shashi2018dont}.

We use T5X large language models~\citep{roberts2022t5x} to train our policies, rewards and preference models.
The T5X models we use are auto-regressive transformers with an encoder-decoder architecture. All the details of the models architecture and the different sizes are provided in the documentation~\citep{roberts2022t5x}.
For the policy model, we use a \emph{large} (L) encoder-decoder model ($770M$ parameters).
For the preference and reward models, we use an \emph{XL} encoder-decoder model ($3B$ parameters). To train reward and preference models we use the same losses and protocol as~\citet{munos2023nash}. For summarisation, we initialise our policy with a T5X-L model and fine-tune it with supervised learning using the OpenAI dataset described by~\citet{stiennon2020learning}. We call this supervised fine-tuned model the SFT. All our policies for summarisation are initialised with this SFT checkpoint. 

Our evaluation pipeline is based upon the use of PaLM2~\citep{anil2023palm} as a judge for side-by-side comparisons. We sample responses for each of the policies trained by each algorithm from a test set of prompts, and ask PaLM2 to pick which one is better. We use validation and test prompts from the XSum dataset~\citep{shashi2018dont} for evaluation for the summarisation task, which is the same procedure used by~\citet{munos2023nash}. The evaluation prompt we use for the side-by-side comparison is:

\texttt{\small You are an expert summary rater. Given a piece of text and two of its possible summaries, output 1 or 2 to indicate which summary is better.}

\texttt{\small Text - <text>, Summary 1 - <summary1>, Summary 2 - <summary2>}

\texttt{\small Preferred Summary -}

We use cloud Tensor Processing Units \citep[TPUs; ][]{Jouppi2023TPUVA} in their version $5e$ for our hardware compute, either in configurations of $2 \times 4$ devices for training offline experiments, or $4 \times 4$ devices for online experiments. This setup typically yields speed of around $0.25$ training steps per second ($24$ hours per $20,000$ steps).  We run our experiments with default parameters $10^{-4}$ for the learning rate, and a default total of $30,000$ training steps, using a batch size of $32$. The $\tau$ factor is held constant throughout training, and we do not employ any warmup steps. We use the AdaFactor \citep{Shazeer2018AdafactorAL} optimizer with decay set to $0.8$.

\vspace{-.5\baselineskip}
\subsection{Main results}
\vspace{-.25\baselineskip}
\label{subsec:main-results}
In this section, we present side-by-side evaluation scores between the following online algorithms: RL, IPO, DPO, SLiC, IPO-MD and Nash-MD-PG. Table~\ref{results:pref_sum} presents the side by side scores for the summarisation task. The checkpoints we evaluate are the best checkpoints except for RL that we use as a baseline of comparison to find the best checkpoints. The RL checkpoint is fixed and was chosen following the protocol of~\citep{munos2023nash} after sweeping over $6$ values of $\tau$ ($\{0.01, 0.02, 0.05, 0.1, 0.15, 0.2\}$) and comparing the performance against the SFT checkpoint after $10k$ learner steps. To find the best checkpoints for the other algorithms, we evaluate every checkpoint of each algorithm against the RL checkpoint (over $2000$ prompts sampled from a validation split) at different learning steps values (we checkpoint every $2000$ learner steps for a total of $30k$ learner steps), regularisation parameter $\tau$ (we sweep over $5$ values $\{0.1, 0.5, 1.0, 5.0, 10.0\}$) and also $\MixtureParam$ for IPO-MD and Nash-MD-PG (we sweep over $2$ values $0.125$ and $0.25$) and we take the best checkpoint. After finding the best checkpoint for every algorithm (see App.~\ref{app:hypers}), we re-run each method for 3 different seeds using the best hyperparamters. We then perform 9 side-by-side evaluation (i.e., $3\times3$ 1vs1 evaluations between each of the 3 seeds for each pair of methods) using $2000$ prompts from a different validation split for each comparison. We report mean and standard deviation across these 9 comparisons.

On the summarisation task, looking only at the mean the best algorithm is IPO as it beats all the other algorithms on a side-by-side comparison. However, once we take into consideration the standard deviation IPO and IPO-MD's performance becomes statistically indistinguishable, with both algorithms consistently beating all the other algorithms. This shows that those algorithms are indeed robust and are closer to a Nash optimum than the other algorithms. Those results are limited to a summarisation task and more experiments should be conducted to validate these results on a general conversational agent. However, we do think that summarisation is a good test bed to showcase the quality of human alignment algorithms because it is a complex and high-in-demand task.

\vspace{-.5\baselineskip}
\subsection{Ablations and Additional Results on Summarisation}
\vspace{-.25\baselineskip}
\label{subsec:ablations}
\paragraph{Sweep on the regularisation parameter $\tau$:}

\cref{fig:beta-sweep} sweeps the regularisation parameter for IPO and DPO. It's interesting to see for small values of regularisation IPO and DPO behave very similarly, but for larger values the score for IPO decays much faster. This matches the findings in \citet{azar2023general} that show that IPO has a much stronger regularisation effect than DPO as $\tau$ gets larger.

\begin{figure}[h]
    \centering
    \vspace{-.5\baselineskip}
    \includegraphics[width=.45\textwidth]{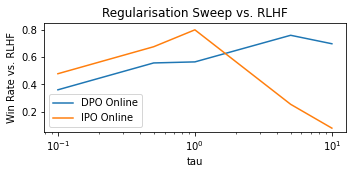}
    \vspace{-1\baselineskip}
    \caption{Sweep on $\tau$ vs. RL on the \emph{summarisation} task.
    \vspace{-.5\baselineskip}
    }
    \label{fig:beta-sweep}
\end{figure}

\vspace{-\baselineskip}
\paragraph{Learning steps curve}
\cref{fig:online-checkpoint-sweep} shows the performance against RL for IPO Online as it trains. We can see that as the regularisation gets stronger, more training time is required to reach the best performance.

\begin{figure}[h!]
    \vspace{-.5\baselineskip}
    \centering
    \includegraphics[width=.45\textwidth]{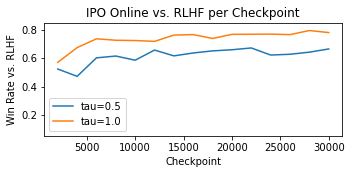}
    \vspace{-\baselineskip}
    \caption{Training Progress of Online IPO (as a function of number of steps) on the \emph{summarisation} task.
    \vspace{-.5\baselineskip}
    }
    \label{fig:online-checkpoint-sweep}
\end{figure}

\section{Conclusion}

In this paper, we have identified several factors of variation, such as contrastivity, online/offline and regularised-sampling, between two recently proposed algorithms for preference optimisation, IPO and Nash-MD-PG. In doing so, we have introduced two new algorithms, online-IPO and IPO-MD, that combine different strengths of these existing algorithms, namely the loss function of IPO with the online sampling and regularised data distribution of Nash-MD. Theoretical analysis reveals a surprising equivalence at the level of expected update between online-IPO and self-play in a regularised two-player preference optimisation game. This important property is not possessed by online-DPO. Finally, our empirical investigation on a summarisation task also reveals that IPO-MD and online-IPO are promising approaches to preference optimisation at scale as they are the most robust algorithms. At the moment, our work is restricted to model of size $770M$ on a single task, future works will consist to scale our approach to a full conversational agent using a larger model ($100+$ billions parameter). 

\textbf{Acknowledgements.} We are grateful for the collaborative environment at Google DeepMind. We would like to thank Shantanu Thakoor, Will Dabney, Doina Precup, Mohammad Gheshlaghi Azar, Olivier Bachem, Sertan Girgin, Matt Hoffman, Nikola Momchev, Bobak Shahriari, and Piotr Stanczyk.

\bibliographystyle{icml2024}

\begin{thebibliography}{54}
\providecommand{\natexlab}[1]{#1}
\providecommand{\url}[1]{\texttt{#1}}
\expandafter\ifx\csname urlstyle\endcsname\relax
  \providecommand{\doi}[1]{doi: #1}\else
  \providecommand{\doi}{doi: \begingroup \urlstyle{rm}\Url}\fi

\bibitem[Amodei et~al.(2016)Amodei, Olah, Steinhardt, Christiano, Schulman, and
  Man{\'e}]{Amodei2016ConcretePI}
Amodei, D., Olah, C., Steinhardt, J., Christiano, P., Schulman, J., and
  Man{\'e}, D.
\newblock Concrete problems in {AI} safety.
\newblock \emph{arXiv}, 2016.

\bibitem[Anil et~al.(2023)Anil, Dai, Firat, Johnson, Lepikhin, Passos, Shakeri,
  Taropa, Bailey, Chen, Chu, Clark, Shafey, Huang, Meier-Hellstern, Mishra,
  Moreira, Omernick, Robinson, Ruder, Tay, Xiao, Xu, Zhang, Abrego, Ahn,
  Austin, Barham, Botha, Bradbury, Brahma, Brooks, Catasta, Cheng, Cherry,
  Choquette-Choo, Chowdhery, Crepy, Dave, Dehghani, Dev, Devlin, Díaz, Du,
  Dyer, Feinberg, Feng, Fienber, Freitag, Garcia, Gehrmann, Gonzalez, Gur-Ari,
  Hand, Hashemi, Hou, Howland, Hu, Hui, Hurwitz, Isard, Ittycheriah, Jagielski,
  Jia, Kenealy, Krikun, Kudugunta, Lan, Lee, Lee, Li, Li, Li, Li, Li, Lim, Lin,
  Liu, Liu, Maggioni, Mahendru, Maynez, Misra, Moussalem, Nado, Nham, Ni,
  Nystrom, Parrish, Pellat, Polacek, Polozov, Pope, Qiao, Reif, Richter, Riley,
  Ros, Roy, Saeta, Samuel, Shelby, Slone, Smilkov, So, Sohn, Tokumine, Valter,
  Vasudevan, Vodrahalli, Wang, Wang, Wang, Wang, Wieting, Wu, Xu, Xu, Xue, Yin,
  Yu, Zhang, Zheng, Zheng, Zhou, Zhou, Petrov, and Wu]{anil2023palm}
Anil, R., Dai, A.~M., Firat, O., Johnson, M., Lepikhin, D., Passos, A.,
  Shakeri, S., Taropa, E., Bailey, P., Chen, Z., Chu, E., Clark, J.~H., Shafey,
  L.~E., Huang, Y., Meier-Hellstern, K., Mishra, G., Moreira, E., Omernick, M.,
  Robinson, K., Ruder, S., Tay, Y., Xiao, K., Xu, Y., Zhang, Y., Abrego, G.~H.,
  Ahn, J., Austin, J., Barham, P., Botha, J., Bradbury, J., Brahma, S., Brooks,
  K., Catasta, M., Cheng, Y., Cherry, C., Choquette-Choo, C.~A., Chowdhery, A.,
  Crepy, C., Dave, S., Dehghani, M., Dev, S., Devlin, J., Díaz, M., Du, N.,
  Dyer, E., Feinberg, V., Feng, F., Fienber, V., Freitag, M., Garcia, X.,
  Gehrmann, S., Gonzalez, L., Gur-Ari, G., Hand, S., Hashemi, H., Hou, L.,
  Howland, J., Hu, A., Hui, J., Hurwitz, J., Isard, M., Ittycheriah, A.,
  Jagielski, M., Jia, W., Kenealy, K., Krikun, M., Kudugunta, S., Lan, C., Lee,
  K., Lee, B., Li, E., Li, M., Li, W., Li, Y., Li, J., Lim, H., Lin, H., Liu,
  Z., Liu, F., Maggioni, M., Mahendru, A., Maynez, J., Misra, V., Moussalem,
  M., Nado, Z., Nham, J., Ni, E., Nystrom, A., Parrish, A., Pellat, M.,
  Polacek, M., Polozov, A., Pope, R., Qiao, S., Reif, E., Richter, B., Riley,
  P., Ros, A.~C., Roy, A., Saeta, B., Samuel, R., Shelby, R., Slone, A.,
  Smilkov, D., So, D.~R., Sohn, D., Tokumine, S., Valter, D., Vasudevan, V.,
  Vodrahalli, K., Wang, X., Wang, P., Wang, Z., Wang, T., Wieting, J., Wu, Y.,
  Xu, K., Xu, Y., Xue, L., Yin, P., Yu, J., Zhang, Q., Zheng, S., Zheng, C.,
  Zhou, W., Zhou, D., Petrov, S., and Wu, Y.
\newblock {PaLM} 2 technical report, 2023.

\bibitem[Azar et~al.(2023)Azar, Rowland, Piot, Guo, Calandriello, Valko, and
  Munos]{azar2023general}
Azar, M.~G., Rowland, M., Piot, B., Guo, D., Calandriello, D., Valko, M., and
  Munos, R.
\newblock A general theoretical paradigm to understand learning from human
  preferences.
\newblock \emph{arXiv}, 2023.

\bibitem[Bai et~al.(2022{\natexlab{a}})Bai, Jones, Ndousse, Askell, Chen,
  DasSarma, Drain, Fort, Ganguli, Henighan, Joseph, Kadavath, Kernion, Conerly,
  El-Showk, Elhage, Hatfield-Dodds, Hernandez, Hume, Johnston, Kravec, Lovitt,
  Nanda, Olsson, Amodei, Brown, Clark, McCandlish, Olah, Mann, and
  Kaplan]{bai2022training}
Bai, Y., Jones, A., Ndousse, K., Askell, A., Chen, A., DasSarma, N., Drain, D.,
  Fort, S., Ganguli, D., Henighan, T., Joseph, N., Kadavath, S., Kernion, J.,
  Conerly, T., El-Showk, S., Elhage, N., Hatfield-Dodds, Z., Hernandez, D.,
  Hume, T., Johnston, S., Kravec, S., Lovitt, L., Nanda, N., Olsson, C.,
  Amodei, D., Brown, T., Clark, J., McCandlish, S., Olah, C., Mann, B., and
  Kaplan, J.
\newblock Training a helpful and harmless assistant with reinforcement learning
  from human feedback.
\newblock \emph{arXiv}, 2022{\natexlab{a}}.

\bibitem[Bai et~al.(2022{\natexlab{b}})Bai, Kadavath, Kundu, Askell, Kernion,
  Jones, Chen, Goldie, Mirhoseini, McKinnon, Chen, Olsson, Olah, Hernandez,
  Drain, Ganguli, Li, Tran-Johnson, Perez, Kerr, Mueller, Ladish, Landau,
  Ndousse, Lukoiūtė, Lovitt, Sellitto, Elhage, Schiefer, Mercado, DasSarma,
  Lasenby, Larson, Ringer, Johnston, Kravec, Showk, Fort, Lanham,
  Telleen-Lawton, Conerly, Henighan, Hume, Bowman, Hatfield-Dodds, Mann,
  Amodei, Joseph, McCandlish, Brown, and Kaplan]{Bai2022ConstitutionalAH}
Bai, Y., Kadavath, S., Kundu, S., Askell, A., Kernion, J., Jones, A., Chen, A.,
  Goldie, A., Mirhoseini, A., McKinnon, C., Chen, C., Olsson, C., Olah, C.,
  Hernandez, D., Drain, D., Ganguli, D., Li, D., Tran-Johnson, E., Perez, E.,
  Kerr, J., Mueller, J., Ladish, J., Landau, J., Ndousse, K., Lukoiūtė, K.,
  Lovitt, L., Sellitto, M., Elhage, N., Schiefer, N., Mercado, N., DasSarma,
  N., Lasenby, R., Larson, R., Ringer, S., Johnston, S., Kravec, S., Showk,
  S.~E., Fort, S., Lanham, T., Telleen-Lawton, T., Conerly, T., Henighan,
  T.~J., Hume, T., Bowman, S., Hatfield-Dodds, Z., Mann, B., Amodei, D.,
  Joseph, N., McCandlish, S., Brown, T.~B., and Kaplan, J.
\newblock Constitutional {AI}: Harmlessness from {AI} feedback.
\newblock \emph{arXiv}, 2022{\natexlab{b}}.

\bibitem[Boyd \& Vandenberghe(2004)Boyd and Vandenberghe]{boyd2004convex}
Boyd, S.~P. and Vandenberghe, L.
\newblock \emph{Convex Optimization}.
\newblock Cambridge University Press, 2004.

\bibitem[Bradley \& Terry(1952)Bradley and Terry]{bradley1952rank}
Bradley, R.~A. and Terry, M.~E.
\newblock Rank analysis of incomplete block designs: I. the method of paired
  comparisons.
\newblock \emph{Biometrika}, 39\penalty0 (3/4):\penalty0 324--345, 1952.

\bibitem[Casper et~al.(2023)Casper, Davies, Shi, Gilbert, Scheurer, Rando,
  Freedman, Korbak, Lindner, Freire, et~al.]{casper2023open}
Casper, S., Davies, X., Shi, C., Gilbert, T.~K., Scheurer, J., Rando, J.,
  Freedman, R., Korbak, T., Lindner, D., Freire, P., et~al.
\newblock Open problems and fundamental limitations of reinforcement learning
  from human feedback.
\newblock \emph{arXiv}, 2023.

\bibitem[Christiano et~al.(2017)Christiano, Leike, Brown, Martic, Legg, and
  Amodei]{christiano2017deep}
Christiano, P.~F., Leike, J., Brown, T., Martic, M., Legg, S., and Amodei, D.
\newblock Deep reinforcement learning from human preferences.
\newblock In \emph{Advances in Neural Information Processing Systems}, 2017.

\bibitem[Coste et~al.(2023)Coste, Anwar, Kirk, and Krueger]{Coste2023RewardME}
Coste, T., Anwar, U., Kirk, R., and Krueger, D.~S.
\newblock Reward model ensembles help mitigate overoptimization.
\newblock \emph{arXiv}, 2023.

\bibitem[Dong et~al.(2023)Dong, Xiong, Goyal, Pan, Diao, Zhang, Shum, and
  Zhang]{Dong2023RAFTRR}
Dong, H., Xiong, W., Goyal, D., Pan, R., Diao, S., Zhang, J., Shum, K., and
  Zhang, T.
\newblock {RAFT}: Reward r{A}nked {FineTuning} for generative foundation model
  alignment.
\newblock \emph{arXiv}, 2023.

\bibitem[Eisenstein et~al.(2023)Eisenstein, Nagpal, Agarwal, Beirami, D'Amour,
  Dvijotham, Fisch, Heller, Pfohl, Ramachandran, Shaw, and
  Berant]{Eisenstein2023HelpingOH}
Eisenstein, J., Nagpal, C., Agarwal, A., Beirami, A., D'Amour, A., Dvijotham,
  D., Fisch, A., Heller, K., Pfohl, S.~R., Ramachandran, D., Shaw, P., and
  Berant, J.
\newblock Helping or herding? {R}ward model ensembles mitigate but do not
  eliminate reward hacking.
\newblock \emph{arXiv}, 2023.

\bibitem[Fujimoto et~al.(2019)Fujimoto, Meger, and Precup]{fujimoto2019off}
Fujimoto, S., Meger, D., and Precup, D.
\newblock Off-policy deep reinforcement learning without exploration.
\newblock In \emph{Proceedings of the International Conference on Machine
  Learning}, 2019.

\bibitem[Gao et~al.(2022)Gao, Schulman, and Hilton]{Gao2022ScalingLF}
Gao, L., Schulman, J., and Hilton, J.
\newblock Scaling laws for reward model overoptimization.
\newblock In \emph{Proceedings of the International Conference on Machine
  Learning}, 2022.

\bibitem[Glaese et~al.(2022)Glaese, McAleese, Trebacz, Aslanides, Firoiu,
  Ewalds, Rauh, Weidinger, Chadwick, Thacker, Campbell-Gillingham, Uesato,
  Huang, Comanescu, Yang, See, Dathathri, Greig, Chen, Fritz, Elias, Green,
  Mokrá, Fernando, Wu, Foley, Young, Gabriel, Isaac, Mellor, Hassabis,
  Kavukcuoglu, Hendricks, and Irving]{glaese2022improving}
Glaese, A., McAleese, N., Trebacz, M., Aslanides, J., Firoiu, V., Ewalds, T.,
  Rauh, M., Weidinger, L., Chadwick, M., Thacker, P., Campbell-Gillingham, L.,
  Uesato, J., Huang, P.-S., Comanescu, R., Yang, F., See, A., Dathathri, S.,
  Greig, R., Chen, C., Fritz, D., Elias, J.~S., Green, R., Mokrá, S.,
  Fernando, N., Wu, B., Foley, R., Young, S., Gabriel, I., Isaac, W., Mellor,
  J., Hassabis, D., Kavukcuoglu, K., Hendricks, L.~A., and Irving, G.
\newblock Improving alignment of dialogue agents via targeted human judgements.
\newblock \emph{arXiv}, 2022.

\bibitem[Griffith et~al.(2013)Griffith, Subramanian, Scholz, Isbell, and
  Thomaz]{griffith2013policy}
Griffith, S., Subramanian, K., Scholz, J., Isbell, C.~L., and Thomaz, A.~L.
\newblock Policy shaping: Integrating human feedback with reinforcement
  learning.
\newblock In \emph{Advances in Neural Information Processing Systems}, 2013.

\bibitem[Hejna et~al.(2023)Hejna, Rafailov, Sikchi, Finn, Niekum, Knox, and
  Sadigh]{hejna2023contrastive}
Hejna, J., Rafailov, R., Sikchi, H., Finn, C., Niekum, S., Knox, W.~B., and
  Sadigh, D.
\newblock Contrastive prefence learning: Learning from human feedback without
  {RL}.
\newblock \emph{arXiv}, 2023.

\bibitem[Ivison et~al.(2023)Ivison, Wang, Pyatkin, Lambert, Peters, Dasigi,
  Jang, Wadden, Smith, Beltagy, and Hajishirzi]{Ivison2023CamelsIA}
Ivison, H., Wang, Y., Pyatkin, V., Lambert, N., Peters, M., Dasigi, P., Jang,
  J., Wadden, D., Smith, N.~A., Beltagy, I., and Hajishirzi, H.
\newblock Camels in a changing climate: {E}nhancing {LM} adaptation with {Tulu}
  2.
\newblock \emph{arXiv}, 2023.

\bibitem[Jaques et~al.(2019)Jaques, Ghandeharioun, Shen, Ferguson, Lapedriza,
  Jones, Gu, and Picard]{jaques2019way}
Jaques, N., Ghandeharioun, A., Shen, J.~H., Ferguson, C., Lapedriza, A., Jones,
  N., Gu, S., and Picard, R.
\newblock Way off-policy batch deep reinforcement learning of implicit human
  preferences in dialog.
\newblock \emph{arXiv}, 2019.

\bibitem[Jouppi et~al.(2023)Jouppi, Kurian, Li, Ma, Nagarajan, Nai, Patil,
  Subramanian, Swing, Towles, Young, Zhou, Zhou, and
  Patterson]{Jouppi2023TPUVA}
Jouppi, N.~P., Kurian, G., Li, S., Ma, P.~C., Nagarajan, R., Nai, L., Patil,
  N., Subramanian, S., Swing, A., Towles, B., Young, C., Zhou, X., Zhou, Z.,
  and Patterson, D.~A.
\newblock {TPU v4}: An optically reconfigurable supercomputer for machine
  learning with hardware support for embeddings.
\newblock In \emph{Proceedings of the Annual International Symposium on
  Computer Architecture}, 2023.

\bibitem[Kirk et~al.(2023)Kirk, Mediratta, Nalmpantis, Luketina, Hambro,
  Grefenstette, and Raileanu]{kirk2023understanding}
Kirk, R., Mediratta, I., Nalmpantis, C., Luketina, J., Hambro, E.,
  Grefenstette, E., and Raileanu, R.
\newblock Understanding the effects of {RLHF} on {LLM} generalisation and
  diversity.
\newblock \emph{arXiv}, 2023.

\bibitem[Knox \& Stone(2008)Knox and Stone]{knox2008tamer}
Knox, W.~B. and Stone, P.
\newblock {TAMER}: Training an agent manually via evaluative reinforcement.
\newblock In \emph{Proceedings of the IEEE International Conference on
  Development and Learning}, 2008.

\bibitem[Lee et~al.(2023)Lee, Phatale, Mansoor, Lu, Mesnard, Bishop, Carbune,
  and Rastogi]{lee2023rlaif}
Lee, H., Phatale, S., Mansoor, H., Lu, K., Mesnard, T., Bishop, C., Carbune,
  V., and Rastogi, A.
\newblock {RLAIF}: Scaling reinforcement learning from human feedback with {AI}
  feedback.
\newblock \emph{arXiv}, 2023.

\bibitem[Liu et~al.(2023)Liu, Zhao, Joshi, Khalman, Saleh, Liu, and
  Liu]{liu2023statistical}
Liu, T., Zhao, Y., Joshi, R., Khalman, M., Saleh, M., Liu, P.~J., and Liu, J.
\newblock Statistical rejection sampling improves preference optimization.
\newblock \emph{arXiv}, 2023.

\bibitem[Mnih et~al.(2016)Mnih, Badia, Mirza, Graves, Lillicrap, Harley,
  Silver, and Kavukcuoglu]{Mnih2016asynchronous}
Mnih, V., Badia, A.~P., Mirza, M., Graves, A., Lillicrap, T.~P., Harley, T.,
  Silver, D., and Kavukcuoglu, K.
\newblock Asynchronous methods for deep reinforcement learning.
\newblock In \emph{Proceedings of the International Conference on Machine
  Learning}, 2016.

\bibitem[Munos et~al.(2023)Munos, Valko, Calandriello, Azar, Rowland, Guo,
  Tang, Geist, Mesnard, Michi, Selvi, Girgin, Momchev, Bachem, Mankowitz,
  Precup, and Piot]{munos2023nash}
Munos, R., Valko, M., Calandriello, D., Azar, M.~G., Rowland, M., Guo, D.,
  Tang, Y., Geist, M., Mesnard, T., Michi, A., Selvi, M., Girgin, S., Momchev,
  N., Bachem, O., Mankowitz, D.~J., Precup, D., and Piot, B.
\newblock Nash learning from human feedback.
\newblock \emph{arXiv}, 2023.

\bibitem[Nakano et~al.(2021)Nakano, Hilton, Balaji, Wu, Ouyang, Kim, Hesse,
  Jain, Kosaraju, Saunders, Jiang, Cobbe, Eloundou, Krueger, Button, Knight,
  Chess, and Schulman]{nakano2021webgpt}
Nakano, R., Hilton, J., Balaji, S., Wu, J., Ouyang, L., Kim, C., Hesse, C.,
  Jain, S., Kosaraju, V., Saunders, W., Jiang, X., Cobbe, K., Eloundou, T.,
  Krueger, G., Button, K., Knight, M., Chess, B., and Schulman, J.
\newblock Web{GPT}: Browser-assisted question-answering with human feedback.
\newblock \emph{arXiv}, 2021.

\bibitem[OpenAI(2022)]{chatgpt}
OpenAI.
\newblock Introducing {ChatGPT}, 2022.
\newblock URL \url{https://openai.com/blog/chatgpt}.

\bibitem[Ouyang et~al.(2022)Ouyang, Wu, Jiang, Almeida, Wainwright, Mishkin,
  Zhang, Agarwal, Slama, Ray, Schulman, Hilton, Kelton, Miller, Simens, Askell,
  Welinder, Christiano, Leike, and Lowe]{InstructGPT}
Ouyang, L., Wu, J., Jiang, X., Almeida, D., Wainwright, C.~L., Mishkin, P.,
  Zhang, C., Agarwal, S., Slama, K., Ray, A., Schulman, J., Hilton, J., Kelton,
  F., Miller, L., Simens, M., Askell, A., Welinder, P., Christiano, P., Leike,
  J., and Lowe, R.
\newblock Training language models to follow instructions with human feedback.
\newblock \emph{arXiv}, 2022.

\bibitem[Pan et~al.(2022)Pan, Bhatia, and Steinhardt]{Pan2022TheEO}
Pan, A., Bhatia, K., and Steinhardt, J.
\newblock The effects of reward misspecification: Mapping and mitigating
  misaligned models.
\newblock \emph{arXiv}, 2022.

\bibitem[Pang et~al.(2022)Pang, Padmakumar, Sellam, Parikh, and
  He]{Pang2022RewardGI}
Pang, R.~Y., Padmakumar, V., Sellam, T., Parikh, A.~P., and He, H.
\newblock Reward gaming in conditional text generation.
\newblock In \emph{Annual Meeting of the Association for Computational
  Linguistics}, 2022.

\bibitem[Rafailov et~al.(2023)Rafailov, Sharma, Mitchell, Ermon, Manning, and
  Finn]{rafailov2023direct}
Rafailov, R., Sharma, A., Mitchell, E., Ermon, S., Manning, C.~D., and Finn, C.
\newblock Direct preference optimization: Your language model is secretly a
  reward model.
\newblock In \emph{Advances in Neural Information Processing Systems}, 2023.

\bibitem[Ram{\'e} et~al.(2024)Ram{\'e}, Vieillard, Hussenot, Dadashi, Cideron,
  Bachem, and Ferret]{Rame2024WARMOT}
Ram{\'e}, A., Vieillard, N., Hussenot, L., Dadashi, R., Cideron, G., Bachem,
  O., and Ferret, J.
\newblock {WARM}: On the benefits of weight averaged reward models.
\newblock \emph{arXiv}, 2024.

\bibitem[Roberts et~al.(2022)Roberts, Chung, Levskaya, Mishra, Bradbury, Andor,
  Narang, Lester, Gaffney, Mohiuddin, Hawthorne, Lewkowycz, Salcianu, van Zee,
  Austin, Goodman, Soares, Hu, Tsvyashchenko, Chowdhery, Bastings, Bulian,
  Garcia, Ni, Chen, Kenealy, Clark, Lee, Garrette, Lee-Thorp, Raffel, Shazeer,
  Ritter, Bosma, Passos, Maitin-Shepard, Fiedel, Omernick, Saeta, Sepassi,
  Spiridonov, Newlan, and Gesmundo]{roberts2022t5x}
Roberts, A., Chung, H.~W., Levskaya, A., Mishra, G., Bradbury, J., Andor, D.,
  Narang, S., Lester, B., Gaffney, C., Mohiuddin, A., Hawthorne, C., Lewkowycz,
  A., Salcianu, A., van Zee, M., Austin, J., Goodman, S., Soares, L.~B., Hu,
  H., Tsvyashchenko, S., Chowdhery, A., Bastings, J., Bulian, J., Garcia, X.,
  Ni, J., Chen, A., Kenealy, K., Clark, J.~H., Lee, S., Garrette, D.,
  Lee-Thorp, J., Raffel, C., Shazeer, N., Ritter, M., Bosma, M., Passos, A.,
  Maitin-Shepard, J., Fiedel, N., Omernick, M., Saeta, B., Sepassi, R.,
  Spiridonov, A., Newlan, J., and Gesmundo, A.
\newblock Scaling up models and data with $\texttt{t5x}$ and $\texttt{seqio}$.
\newblock \emph{arXiv}, 2022.

\bibitem[Schulman et~al.(2017)Schulman, Wolski, Dhariwal, Radford, and
  Klimov]{PPO}
Schulman, J., Wolski, F., Dhariwal, P., Radford, A., and Klimov, O.
\newblock Proximal policy optimization algorithms.
\newblock \emph{arXiv}, 2017.

\bibitem[Shashi et~al.(2018)Shashi, Cohen, and Mirella]{shashi2018dont}
Shashi, N., Cohen, S.~B., and Mirella, L.
\newblock {Don’t Give Me the Details, Just the Summary}! {T}opic-aware
  convolutional neural networks for extreme summarization.
\newblock In \emph{Proceedings of the Conference on Empirical Methods in
  Natural Language Processing}, 2018.

\bibitem[Shazeer \& Stern(2018)Shazeer and Stern]{Shazeer2018AdafactorAL}
Shazeer, N.~M. and Stern, M.
\newblock Adafactor: Adaptive learning rates with sublinear memory cost.
\newblock \emph{arXiv}, 2018.

\bibitem[Shin et~al.(2023)Shin, Dragan, and Brown]{Shin2023BenchmarksAA}
Shin, D., Dragan, A.~D., and Brown, D.~S.
\newblock Benchmarks and algorithms for offline preference-based reward
  learning.
\newblock \emph{arXiv}, 2023.

\bibitem[Singhal et~al.(2023)Singhal, Goyal, Xu, and Durrett]{Singhal2023ALW}
Singhal, P., Goyal, T., Xu, J., and Durrett, G.
\newblock A long way to go: Investigating length correlations in {RLHF}.
\newblock \emph{arXiv}, 2023.

\bibitem[Skalse et~al.(2022)Skalse, Howe, Krasheninnikov, and
  Krueger]{Skalse2022DefiningAC}
Skalse, J., Howe, N. H.~R., Krasheninnikov, D., and Krueger, D.
\newblock Defining and characterizing reward gaming.
\newblock In \emph{Advances in Neural Information Processing Systems}, 2022.

\bibitem[Stiennon et~al.(2020)Stiennon, Ouyang, Wu, Ziegler, Lowe, Voss,
  Radford, Amodei, and Christiano]{stiennon2020learning}
Stiennon, N., Ouyang, L., Wu, J., Ziegler, D., Lowe, R., Voss, C., Radford, A.,
  Amodei, D., and Christiano, P.~F.
\newblock Learning to summarize with human feedback.
\newblock In \emph{Advances in Neural Information Processing Systems}, 2020.

\bibitem[Swamy et~al.(2024)Swamy, Dann, Kidambi, Wu, and Agarwal]{Swamy2024AMA}
Swamy, G., Dann, C., Kidambi, R., Wu, Z.~S., and Agarwal, A.
\newblock A minimaximalist approach to reinforcement learning from human
  feedback.
\newblock \emph{arXiv}, 2024.

\bibitem[Tang et~al.(2024)Tang, Guo, Zheng, Calandriello, Munos, Rowland,
  Richemond, Valko, Pires, and Piot]{tang2024generalized}
Tang, Y., Guo, Z.~D., Zheng, Z., Calandriello, D., Munos, R., Rowland, M.,
  Richemond, P.~H., Valko, M., Pires, B.~{\'A}., and Piot, B.
\newblock Generalized preference optimization: A unified approach to offline
  alignment.
\newblock \emph{arXiv preprint arXiv:2402.05749}, 2024.

\bibitem[Touvron et~al.(2023)Touvron, Martin, Stone, Albert, Almahairi, Babaei,
  Bashlykov, Batra, Bhargava, Bhosale, Bikel, Blecher, Ferrer, Chen, Cucurull,
  Esiobu, Fernandes, Fu, Fu, Fuller, Gao, Goswami, Goyal, Hartshorn, Hosseini,
  Hou, Inan, Kardas, Kerkez, Khabsa, Kloumann, Korenev, Koura, Lachaux, Lavril,
  Lee, Liskovich, Lu, Mao, Martinet, Mihaylov, Mishra, Molybog, Nie, Poulton,
  Reizenstein, Rungta, Saladi, Schelten, Silva, Smith, Subramanian, Tan, Tang,
  Taylor, Williams, Kuan, Xu, Yan, Zarov, Zhang, Fan, Kambadur, Narang,
  Rodriguez, Stojnic, Edunov, and Scialom]{Touvron2023Llama2O}
Touvron, H., Martin, L., Stone, K.~R., Albert, P., Almahairi, A., Babaei, Y.,
  Bashlykov, N., Batra, S., Bhargava, P., Bhosale, S., Bikel, D.~M., Blecher,
  L., Ferrer, C.~C., Chen, M., Cucurull, G., Esiobu, D., Fernandes, J., Fu, J.,
  Fu, W., Fuller, B., Gao, C., Goswami, V., Goyal, N., Hartshorn, A.~S.,
  Hosseini, S., Hou, R., Inan, H., Kardas, M., Kerkez, V., Khabsa, M.,
  Kloumann, I.~M., Korenev, A.~V., Koura, P.~S., Lachaux, M.-A., Lavril, T.,
  Lee, J., Liskovich, D., Lu, Y., Mao, Y., Martinet, X., Mihaylov, T., Mishra,
  P., Molybog, I., Nie, Y., Poulton, A., Reizenstein, J., Rungta, R., Saladi,
  K., Schelten, A., Silva, R., Smith, E.~M., Subramanian, R., Tan, X., Tang,
  B., Taylor, R., Williams, A., Kuan, J.~X., Xu, P., Yan, Z., Zarov, I., Zhang,
  Y., Fan, A., Kambadur, M., Narang, S., Rodriguez, A., Stojnic, R., Edunov,
  S., and Scialom, T.
\newblock Llama 2: Open foundation and fine-tuned chat models.
\newblock \emph{arXiv}, 2023.

\bibitem[Tunstall et~al.(2023)Tunstall, Beeching, Lambert, Rajani, Rasul,
  Belkada, Huang, von Werra, Fourrier, Habib, Sarrazin, Sanseviero, Rush, and
  Wolf]{tunstall2023zephyr}
Tunstall, L., Beeching, E., Lambert, N., Rajani, N., Rasul, K., Belkada, Y.,
  Huang, S., von Werra, L., Fourrier, C., Habib, N., Sarrazin, N., Sanseviero,
  O., Rush, A.~M., and Wolf, T.
\newblock Zephyr: Direct distillation of {LM} alignment.
\newblock \emph{arXiv}, 2023.

\bibitem[V{\"o}lske et~al.(2017)V{\"o}lske, Potthast, Syed, and Stein]{TLDR}
V{\"o}lske, M., Potthast, M., Syed, S., and Stein, B.
\newblock {TL};{DR}: Mining {R}eddit to learn automatic summarization.
\newblock In \emph{Proceedings of the Workshop on New Frontiers in
  Summarization}. Association for Computational Linguistics, 2017.

\bibitem[Wang et~al.(2023)Wang, Jiang, Yang, Liu, and Chen]{Wang2023BeyondRK}
Wang, C., Jiang, Y., Yang, C., Liu, H., and Chen, Y.
\newblock Beyond reverse {KL}: Generalizing direct preference optimization with
  diverse divergence constraints.
\newblock \emph{arXiv}, 2023.

\bibitem[Warnell et~al.(2018)Warnell, Waytowich, Lawhern, and
  Stone]{warnell2018deep}
Warnell, G., Waytowich, N., Lawhern, V., and Stone, P.
\newblock Deep {TAMER}: Interactive agent shaping in high-dimensional state
  spaces.
\newblock In \emph{Proceedings of the AAAI Conference on Artificial
  Intelligence}, 2018.

\bibitem[Wortsman et~al.(2022)Wortsman, Ilharco, Gadre, Roelofs, Gontijo-Lopes,
  Morcos, Namkoong, Farhadi, Carmon, Kornblith, and
  Schmidt]{Wortsman2022ModelSA}
Wortsman, M., Ilharco, G., Gadre, S.~Y., Roelofs, R., Gontijo-Lopes, R.,
  Morcos, A.~S., Namkoong, H., Farhadi, A., Carmon, Y., Kornblith, S., and
  Schmidt, L.
\newblock Model soups: averaging weights of multiple fine-tuned models improves
  accuracy without increasing inference time.
\newblock In \emph{Proceedings of the International Conference on Machine
  Learning}, 2022.

\bibitem[Wu et~al.(2019)Wu, Tucker, and Nachum]{wu2019behavior}
Wu, Y., Tucker, G., and Nachum, O.
\newblock Behavior regularized offline reinforcement learning.
\newblock \emph{arXiv}, 2019.

\bibitem[Yuan et~al.(2024)Yuan, Pang, Cho, Sukhbaatar, Xu, and
  Weston]{yuan2024selfrewarding}
Yuan, W., Pang, R.~Y., Cho, K., Sukhbaatar, S., Xu, J., and Weston, J.
\newblock Self-rewarding language models, 2024.

\bibitem[Yuan et~al.(2023)Yuan, Yuan, Tan, Wang, Huang, and
  Huang]{Yuan2023RRHFRR}
Yuan, Z., Yuan, H., Tan, C., Wang, W., Huang, S., and Huang, F.
\newblock {RRHF}: Rank responses to align language models with human feedback
  without tears.
\newblock \emph{arXiv}, 2023.

\bibitem[Zhao et~al.(2023)Zhao, Joshi, Liu, Khalman, Saleh, and
  Liu]{zhao2023slichf}
Zhao, Y., Joshi, R., Liu, T., Khalman, M., Saleh, M., and Liu, P.~J.
\newblock {SLiC-HF}: Sequence likelihood calibration with human feedback.
\newblock \emph{arXiv}, 2023.

\bibitem[Zhuang \& Hadfield-Menell(2020)Zhuang and
  Hadfield-Menell]{Zhuang2021ConsequencesOM}
Zhuang, S. and Hadfield-Menell, D.
\newblock Consequences of misaligned {AI}.
\newblock In \emph{Advances in Neural Information Processing Systems}, 2020.

\end{thebibliography}

\newpage
\appendix
\onecolumn
\section*{\centering APPENDICES}

\section{Related Work}

\textbf{RLHF.} Reinforcement learning from human feedback, as introduced in \citet{christiano2017deep} (see also \citet{bai2022training,InstructGPT}) and often based on proximal policy optimisation \citep{PPO}, is a critical element of making large language models helpful and aligned with preferences of human operators. While in itself it typically does not result in improved benchmark performance \citep{Touvron2023Llama2O}, RLHF is nonetheless key to satisfying human-mediated interactions such as dialogue \citep{ nakano2021webgpt, InstructGPT}. The complexity of the RLHF procedure \citep{casper2023open}, which can also be accomplished by multiple reinforcement learning algorithms such as actor-critic \cite{Mnih2016asynchronous, glaese2022improving}, has led to searching for algorithmic alternatives \cite{Dong2023RAFTRR,Yuan2023RRHFRR,zhao2023slichf}.

\textbf{Recent developments in policy optimisation.} In the special case of an additional Bradley-Terry model \citep{bradley1952rank} assumption for the human reward model, reinforcement learning has been found redundant; this allows for casting the problem of RLHF as a supervised one \citep{rafailov2023direct}. Recent developments have focused on scaling the performance of such direct policy optimisation (DPO) methods \citep{tunstall2023zephyr, Ivison2023CamelsIA}, as well as generalising its mathematical formulation \cite{azar2023general, Wang2023BeyondRK, tang2024generalized}. One of the key issues with direct policy optimisation - and RLHF in general - resides in their propensity to game or \emph{hack} rewards \citep{Amodei2016ConcretePI, Skalse2022DefiningAC,Pan2022TheEO,Pang2022RewardGI} and become overoptimised, or under-regularised \citep{Gao2022ScalingLF, Singhal2023ALW, kirk2023understanding}, which can be mitigated e.g. by using ensembling techniques \citep{Wortsman2022ModelSA, Eisenstein2023HelpingOH,  Coste2023RewardME,Rame2024WARMOT}.
Rather than a reinforcement versus supervised learning dichotomy, it is the distinction between \emph{online} and \emph{offline} \citep{jaques2019way} methods that seems more relevant in practice, as an online policy's generations might start deviating substantially from the original dataset, leading to distribution shifts \cite{Zhuang2021ConsequencesOM,Shin2023BenchmarksAA}. Finally, alignment can also result from the self-play form of a two-player game, and not just single-policy optimisation. This perspective, taken in \citet{munos2023nash, Swamy2024AMA} has the added benefit of encompassing both online and offline settings, enabling smooth interpolation between them via a hyperparameter. In a similar vein, DPO has been shown to be able and improve thanks to iterated successive rounds \citep{yuan2024selfrewarding}, expanding on known machine-critic alignment methods such as reinforcement learning from AI feedback \cite{Bai2022ConstitutionalAH, lee2023rlaif}.

\section{Additional Experimental Results}

\subsection{Regularisation Sweep for Online and Offline}

\cref{fig:beta-sweep-both} and \cref{fig:beta-sweep-both-sft} show a sweep over the regularisation parameter for Online and Offline DPO and IPO vs. RLHF and SFT respectively. It is interesting to note that the online versions significantly outperform the offline versions. This is understandable as this setting favours tremendously online methods. Indeed, the starting point is an already fine-tuned policy on summarisation data. Therefore, for online methods the first checkpoint is already able to sample good summaries which make it very easy to obtain good rewards/preferences and from there optimise either the reward/preference model.  

\begin{figure}[h]
    \centering
    \includegraphics[width=.35\textwidth]{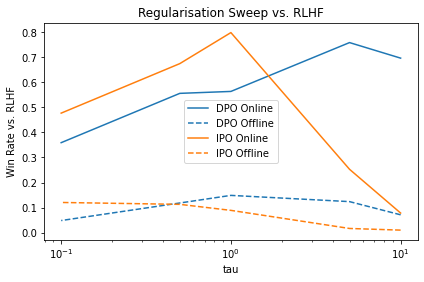}
    \caption{Sweep on Regularisation parameter vs. RL on the \emph{summarisation} task.}
    \label{fig:beta-sweep-both}
\end{figure}

\begin{figure}[h]
    \centering
    \includegraphics[width=.35\textwidth]{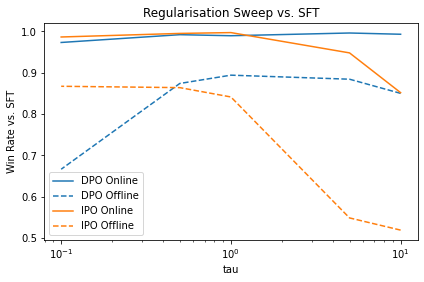}
    \caption{Sweep on Regularisation parameter vs. SFT on the \emph{summarisation} task.}
    \label{fig:beta-sweep-both-sft}
\end{figure}

\subsection{Mixing ratio curve}

Here in \cref{fig:online-mixture-sweep} we show a sweep over the mixing ratio $\MixtureParam$ for IPO-MD and how it affects its win-rate over the RL baseline in summarisation. We draw the curves for a different learning rate and different learning steps than the optimal checkpoint of IPO-MD to show that most of the time the mixing ratio still help improve the performance.

\begin{figure}[h]
    \centering
    \includegraphics[width=.35\textwidth]{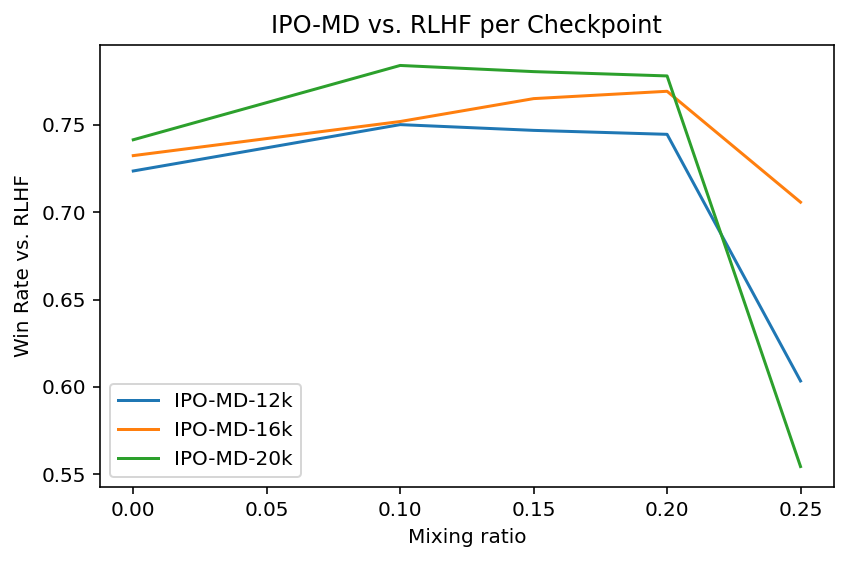}
    \caption{IPO-MD winning rate. The evolving parameter is the mixture ratio $\MixtureParam$ while we fix the learning rate to $3\cdot10^{-5}$ and the regularisation parameter to $\tau=1$. We plot this curves for $3$ different learning steps $12k$, $16k$ and $20k$.}
    \label{fig:online-mixture-sweep}
\end{figure}

\subsection{Best hyperparameters found}\label{app:hypers}
We report here the optimal $\tau$, learning rate \texttt{lr} and, where applicable, mixture ratio $\beta$ used to obtain each algorithm's best performance in Table~\ref{results:pref_sum}.

\begin{center}
{
\setlength\tabcolsep{4.7pt}
\begin{tabular}{l|c|c|c|}
 &  $\tau$ &  \texttt{lr}  &   $\beta$  \\
\midrule
 RL & 0.05 & $10^{-4}$ & N/A \\
 IPO & 1.0 & $10^{-4}$ & N/A \\
 DPO & 5.0 & $10^{-4}$ & N/A \\
 SLiC & 10.0 & $10^{-4}$ & N/A \\
 IPO-MD & 1.0 & $10^{-4}$ & 0.125 \\
 Nash-MD-PG & 0.008 & $3\cdot10^{-5}$ &	0.125 \\
\end{tabular}
}
\end{center}

\section{Proofs}\label{sec:proofs}\label{sec:gradients}

\thmOnlineIPOSelfPlay*

\begin{proof}
    We calculate expressions for the update directions directly, assuming that the policy $\pi$ is parametrised via a vector $\phi$.
    
    \textbf{Online IPO update.}
    The update direction for online IPO is given by the negative of the derivative:
    \begin{align*}
        & \nabla_\phi \mathop{\mathbb{E}}_{
        \substack{Y, Y' \sim \textcolor{red}{\texttt{SG}[\pi]} \\ Y^+, Y^- \sim \lambda_p(Y, Y')}}\bigg\lbrack \left( \log \bigg( \frac{\pi(Y)\piref(Y')}{\pi(Y') \piref(Y)} \bigg) -  \tau^{-1}/2 \right)^2 \bigg\rbrack  \\
        \propto & \sum_{y, y' \in \mathcal{Y}} \pi(y) \pi(y') p(y \succ y') \nabla_\phi \bigg(\bigg( \log\bigg(\frac{\pi(y) \piref(y')}{\pi(y') \piref(y)}\bigg) - \tau^{-1}/2 \bigg)^2 \bigg) \\
        \propto & \sum_{y, y' \in \mathcal{Y}} \pi(y) \pi(y') p(y \succ y')\bigg( \log\bigg(\frac{\pi(y) \piref(y')}{\pi(y') \piref(y)}\bigg) - \tau^{-1}/2 \bigg) \nabla_\phi \log(\pi(y)/\pi(y')) \, .
    \end{align*}
    We can simplify the above by first considering the terms with a factor of $\tau^{-1}$:
    \begin{align*}
       & - \tau^{-1}/2 \sum_{y, y' \in \mathcal{Y}} \pi(y) \pi(y') p(y \succ y') (\nabla_\phi \log \pi(y) - \nabla_\phi \log \pi(y')) \\
        = & - \tau^{-1}/2 \left(  \sum_{y \in \mathcal{Y}} \pi(y) p(y \succ \pi) \nabla_\phi \log \pi(y) - \sum_{y' \in \mathcal{Y}} p(\pi \succ y') \pi(y') \nabla_\phi \log \pi(y') \right) \\
        = & - \tau^{-1}/2 \left(  \sum_{y \in \mathcal{Y}} \pi(y) p(y \succ \pi) \nabla_\phi \log \pi(y) - \sum_{y' \in \mathcal{Y}} (1 - p(y' \succ \pi)) \pi(y') \nabla_\phi \log \pi(y') \right) \\
        = & - \tau^{-1} \sum_{y \in \mathcal{Y}} \pi(y) p(y \succ \pi) \nabla_\phi \log \pi(y) \, .
    \end{align*}
    Now considering the terms not involving $\tau$:
    \begin{align*}
        & \sum_{y, y' \in \mathcal{Y}} \pi(y) \pi(y') p(y \succ y') \Bigg(\log\bigg( \frac{\pi(y)}{\piref(y)} \bigg) - \log \bigg( \frac{\pi(y')}{\piref(y')} \bigg)  \Bigg) (\nabla_\phi \log \pi(y) - \nabla_\phi \log \pi(y')) \\
        = &  \sum_{y, y' \in \mathcal{Y}} \pi(y) \pi(y') p(y \succ y') \log\bigg( \frac{\pi(y)}{\piref(y)} \bigg)\nabla_\phi \log \pi(y)  \\
        & \qquad + \sum_{y, y' \in \mathcal{Y}} \pi(y) \pi(y') p(y \succ y') \log\bigg( \frac{\pi(y')}{\piref(y')} \bigg)\nabla_\phi \log \pi(y') \\
        & \qquad - \sum_{y, y' \in \mathcal{Y}} \pi(y) \pi(y') p(y \succ y') \log\bigg( \frac{\pi(y)}{\piref(y)} \bigg)\nabla_\phi \log \pi(y') \\
        & \qquad - \sum_{y, y' \in \mathcal{Y}} \pi(y) \pi(y') p(y \succ y') \log\bigg( \frac{\pi(y')}{\piref(y')} \bigg)\nabla_\phi \log \pi(y) \\
        = &  \sum_{y, y' \in \mathcal{Y}} \pi(y) p(y \succ \pi) \log\bigg( \frac{\pi(y)}{\piref(y)} \bigg)\nabla_\phi \log \pi(y) + \sum_{y, y' \in \mathcal{Y}}  \pi(y') p(\pi \succ y') \log\bigg( \frac{\pi(y')}{\piref(y')} \bigg)\nabla_\phi \log \pi(y') \\
        & \qquad - \sum_{y, y' \in \mathcal{Y}} \pi(y) \pi(y') p(y \succ y') \log\bigg( \frac{\pi(y)}{\piref(y)} \bigg)\nabla_\phi \log \pi(y') \\
        & \qquad - \sum_{y, y' \in \mathcal{Y}} \pi(y) \pi(y') (1 - p(y \succ y')) \log\bigg( \frac{\pi(y)}{\piref(y)} \bigg)\nabla_\phi \log \pi(y') \\
        = &  \sum_{y, y' \in \mathcal{Y}} \pi(y) p(y \succ \pi) \log\bigg( \frac{\pi(y)}{\piref(y)} \bigg)\nabla_\phi \log \pi(y) + \sum_{y, y' \in \mathcal{Y}}  \pi(y') (1 - p(y \succ \pi)) \log\bigg( \frac{\pi(y')}{\piref(y')} \bigg)\nabla_\phi \log \pi(y') \\
        & \qquad - \sum_{y, y' \in \mathcal{Y}} \pi(y) \pi(y')  \log\bigg( \frac{\pi(y)}{\piref(y)} \bigg)\nabla_\phi \log \pi(y') \\
        = & \sum_{y, y' \in \mathcal{Y}}  \pi(y') \log\bigg( \frac{\pi(y')}{\piref(y')} \bigg)\nabla_\phi \log \pi(y') \\
        & \qquad - \sum_{y, y' \in \mathcal{Y}} \pi(y) \pi(y')  \log\bigg( \frac{\pi(y)}{\piref(y)} \bigg)\nabla_\phi \log \pi(y') \\
        = & \sum_{y, y' \in \mathcal{Y}}  \pi(y') \log\bigg( \frac{\pi(y')}{\piref(y')} \bigg)\nabla_\phi \log \pi(y') \, .
    \end{align*}
    
    \textbf{Self-play update.}
    Self-play leads to the update direction given by
    \begin{align*}
        & \nabla_\phi \Bigg[ \sum_{y,y' \in \mathcal{Y}} \pi(y) \texttt{SG}[\pi(y')] p(y \succ y') 
         - \tau \sum_{y \in \mathcal{Y}} \pi(y) \log(\pi(y) / \piref(y))  \Bigg] \\
         = & \sum_{y,y' \in \mathcal{Y}} \pi(y)  \pi(y') p(y \succ y') \nabla_\phi \log \pi(y)
         - \tau \sum_{y \in \mathcal{Y}} \pi(y) \log(\pi(y) / \piref(y)) \nabla_\phi \log \pi(y)  \\
         = & \sum_{y} \pi(y)  p(y \succ \pi) \nabla_\phi \log \pi(y)
         - \tau \sum_{y \in \mathcal{Y}} \pi(y) \log(\pi(y) / \piref(y)) \nabla_\phi \log \pi(y) \, .
    \end{align*}
    Therefore the expected online IPO update direction is exactly the same as that of self-play.
\end{proof}

\propIPOGradients*

\begin{proof}
The gradient of Nash-MD-PG($\MixtureParam$) is:
\begin{eqnarray*}
g_{\mbox{{\tiny Nash-MD-PG}}(\MixtureParam)} &=& - {\mathbb E}_{y\sim\pi, y'\sim\pi'}\left[\nabla\log\pi(y)\left( p(y\succ y')-\frac 12 \tau\log\frac{\pi(y)}{\piref(y)}\right)\right] \\
&=& - {\mathbb E}_{y\sim\pi}\Big[
\underbrace{\nabla\log\pi(y)\left( p(y\succ \pi')-\frac 12-\tau\log\frac{\pi(y)}{\piref(y)}\right)}_{g(y)}\Big] \\
\end{eqnarray*}

where we write $\pi'=(\pi)^{1-\MixtureParam}(\piref)^\MixtureParam$.

Now the gradient of IPO-MD($\MixtureParam$) is:
\begin{eqnarray*}
g_{\mbox{{\tiny IPO-MD}}(\MixtureParam)} &=& {\mathbb E}_{y, y'\sim\pi'}\left[ p(y\succ y') \nabla \left( \rho_\pi(y,y') - \frac{1}{2\tau} \right)^2\right] \\
&=& 2 {\mathbb E}_{y, y'\sim\pi'}\left[ p(y\succ y') \left( \rho_\pi(y,y') - \frac{1}{2\tau} \right)\nabla \rho_\pi(y,y') \right], 
\end{eqnarray*}
where $\rho_\pi(y,y') = \log\frac{\pi(y)}{\piref(y)} - \log\frac{\pi(y')}{\piref(y')}$, thus  $\nabla \rho_\pi(y,y') = \nabla \log\pi(y) - \nabla \log\pi(y')$.

Using the anti-symmetry of the preference model, i.e., $p(y\succ y')=1-p(y'\succ y)$, and combining terms, we have that
\begin{eqnarray*}
{\mathbb E}_{y, y'\sim\pi'} \left[p(y\succ y') \nabla \rho_\pi(y,y') \right] &=& {\mathbb E}_{y\sim\pi'} \left[p(y\succ \pi')  \nabla \log\pi(y) \right] - {\mathbb E}_{y'\sim\pi'} \left[p(\pi'\succ y')  \nabla \log\pi(y') \right]\\
 &=& 2 {\mathbb E}_{y\sim\pi'} \left[\nabla \log\pi(y) \left( p(y\succ \pi') -\frac 12\right)  \right].
\end{eqnarray*}
Similarly:
\begin{eqnarray*}
& &{\mathbb E}_{y, y'\sim\pi'} \left[p(y\succ y') \rho_\pi(y,y') \nabla \rho_\pi(y,y') \right] \\
&=& {\mathbb E}_{y,y'\sim\pi'} \left[p(y\succ y')  \left(\log\frac{\pi(y)}{\piref(y)} - \log\frac{\pi(y')}{\piref(y')}\right) [\nabla \log\pi(y)-\nabla \log\pi(y')] \right],\\
&=& {\mathbb E}_{y,y'\sim\pi'} \left[p(y\succ y') 
\log\frac{\pi(y)}{\piref(y)}\nabla \log\pi(y) + (1-p(y'\succ y))\log\frac{\pi(y')}{\piref(y')}\nabla \log\pi(y')\right] \\
& & - {\mathbb E}_{y,y'\sim\pi'} \left[p(y\succ y') 
\log\frac{\pi(y)}{\piref(y)}\nabla \log\pi(y') + (1-p(y'\succ y)) \log\frac{\pi(y')}{\piref(y')}\nabla \log\pi(y)\right]\\
&=& {\mathbb E}_{y,y'\sim\pi'} \left[
\log\frac{\pi(y)}{\piref(y)}\left( \nabla \log\pi(y) - 
\nabla \log\pi(y') \right)\right]\\
\end{eqnarray*}
We deduce that:
\begin{eqnarray*}
g_{\mbox{{\tiny IPO-MD}}(\MixtureParam)} &=&-\frac{2}{\tau} {\mathbb E}_{y\sim\pi'} \left[\nabla \log\pi(y) \left( p(y\succ \pi') -\frac 12\right)  \right] + 2 {\mathbb E}_{y,y'\sim\pi'} \left[
\log\frac{\pi(y)}{\piref(y)}\left( \nabla \log\pi(y) - 
\nabla \log\pi(y') \right)\right]\\
&=&-\frac{2}{\tau} {\mathbb E}_{y,y'\sim\pi'} \left[\nabla \log\pi(y) \left( p(y\succ \pi') -\frac 12\right) - \tau 
\log\frac{\pi(y)}{\piref(y)}\left( \nabla \log\pi(y) - 
\nabla \log\pi(y') \right)\right]\\
&=&-\frac{2}{\tau} {\mathbb E}_{y\sim\pi'} \Big[
\underbrace{\nabla \log\pi(y) \left( p(y\succ \pi') -\frac 12 - \tau \log\frac{\pi(y)}{\piref(y)}\right) }_{=g(y)}
\Big] - 2 {\mathbb E}_{y\sim\pi'} \left[
\log\frac{\pi(y)}{\piref(y)} \right] 
\underbrace{{\mathbb E}_{y\sim\pi'} \left[\nabla \log\pi(y)\right]}_{=0}
\end{eqnarray*}
where we used that ${\mathbb E}_{y\sim\pi'} \left[\nabla \log\pi(y)\right] = \frac{1}{1-\MixtureParam} {\mathbb E}_{y\sim\pi'} \left[\nabla \log\pi'(y)\right]=0$, from the definition of $\pi'$.

We deduce that 
\begin{align*}
    g_{\mbox{{\tiny Nash-MD-PG}}(\MixtureParam)} =& - {\mathbb E}_{y\sim\pi}\left[g(y)\right]\\
    g_{\mbox{{\tiny IPO-MD}}(\MixtureParam)} =& - \frac{2}{\tau}{\mathbb E}_{y\sim\pi'}\left[g(y)\right].  \qedhere
\end{align*}
\end{proof}

\section{Comparison of the variance of contrastive versus non-contrastive gradient estimates}\label{sec:contrastive.vs.non.contrastive}

Define the following gradient estimates based on non-contrastive vs contrastive loss functions:
    \begin{eqnarray*}
        \hat g_{\mbox{{\tiny non-contrastive}}} &=& - \nabla\log\pi(y) \left(  p(y\succ y') -\frac 12 -\tau\log\frac{\pi(y)}{\piref(y)}
        +\tau\log\frac{\pi(y')}{\piref(y')}\right)\\
        \hat g_{\mbox{{\tiny contrastive}}} &=& - \frac 12 \left(\nabla\log\pi(y) -\nabla\log\pi(y')\right) \left( p(y\succ y') -\frac 12 -\tau\log\frac{\pi(y)}{\piref(y)}
        +\tau\log\frac{\pi(y')}{\piref(y')}\right).
    \end{eqnarray*}
These two estimate resemble those of the algorithms Self-Play (as implemented by Nash-MD-PG($\MixtureParam=0$)), which uses a non-contrastive loss, and online-IPO (equivalent to IPO-MD($\MixtureParam=0$)), which uses a contrastive loss, and we have
\begin{eqnarray*}
        g_{\mbox{{\tiny Self-Play}}} &=& {\mathbb E}_{y,y'\sim\pi}\left[\hat g_{\mbox{{\tiny non-contrastive}}}\right]\\
        g_{\mbox{{\tiny online-IPO}}} &=& \frac{2}{\tau} {\mathbb E}_{y,y'\sim\pi}\left[\hat g_{\mbox{{\tiny contrastive}}}\right].
\end{eqnarray*}

We know from the previous result that these two estimates have the same expectation. However their variance may differ. Their respective variance depends on a non-trivial combinaison of the policy representation and the specifics of the preference model. We now state a sufficient condition under which the contrastive gradient estimate has lower variance than its non-contrastive counterpart.
\nocite{hejna2023contrastive}

\begin{proposition}\label{prop:condition.for.variance.reduction}
If the policy representation and the preference model are such that we have
\begin{equation}\label{eqn:condition.variance.reduction}
{\mathbb E}_{y,y'\sim\pi}\left[ \nabla\log\pi(y)\nabla\log\pi(y')f(y,y')^2\right] \geq 0,
\end{equation}
where $f(y,y'):= p(y\succ y') -\frac 12 -\tau\log\frac{\pi(y)}{\piref(y)}
        +\tau\log\frac{\pi(y')}{\piref(y')}$,
then the variance of the contrastive gradient estimate is at least as small as that of the non-contrastive one: $\mbox{Var}(\hat g_{\mbox{{\tiny contrastive}}}) \leq \mbox{Var}(\hat g_{\mbox{{\tiny non-contrastive}}})$.
\end{proposition}

\begin{proof}
Defining the random variables $X_1(y,y'):=- \nabla\log\pi(y) f(y,y')$ and $X_2(y,y'):=\nabla\log\pi(y') f(y,y')$,
we have
\begin{eqnarray*}
        \hat g_{\mbox{{\tiny non-contrastive}}} &=& X_1(y,y')\\
        \hat g_{\mbox{{\tiny contrastive}}} &=& \frac{X_1(y,y')+X_2(y,y')}{2}.
\end{eqnarray*}

Now, let us compare the variance of these estimates when $y$ and $y'$ are independently drawn from the same distribution $\pi$. The variance of the contrastive estimate is
$$\mbox{Var}\left(\frac{X_1+X_2}{2}\right)=\frac{\mbox{Var}(X_1)+\mbox{Var}(X_2)+2 \mbox{Cov}(X_1, X_2)}{4} = \frac 12 \left(\mbox{Var}(X_1) + \mbox{Cov}(X_1, X_2)\right).$$
This variance is lower than $\mbox{Var}(X_1)$ as soon as $\mbox{Cov}(X_1, X_2)$ is negative (this is the principle of antithetic variates for variance reduction).

Using the property that $f(y,y')=-f(y',y)$, and writing $c={\mathbb E}_{y,y'\sim\pi}\left[ \nabla\log\pi(y)f(y,y')\right] = - {\mathbb E}_{y,y'\sim\pi}\left[ \nabla\log\pi(y')f(y,y')\right]$, we have:
\begin{eqnarray*}
\mbox{Cov}(X_1, X_2)&=& {\mathbb E}_{y,y'\sim\pi}\left[ \left(-\nabla\log\pi(y)f(y,y')+c \right) \left(\nabla\log\pi(y')f(y,y')+c\right)\right]\\
&=&{\mathbb E}_{y,y'\sim\pi}\left[ -\nabla\log\pi(y)\nabla\log\pi(y')f(y,y')^2\right] - c^2.
\end{eqnarray*}

We deduce that a sufficient condition for the contrastive estimate to have a lower variance than that of the non-contrastive estimate is:
$${\mathbb E}_{y,y'\sim\pi}\left[ \nabla\log\pi(y)\nabla\log\pi(y')f(y,y')^2\right] + \left({\mathbb E}_{y,y'\sim\pi}\left[ \nabla\log\pi(y)f(y,y')\right]\right)^2\geq 0.$$
This condition is true as soon as Equation~\eqref{eqn:condition.variance.reduction} is satisfied.
\end{proof}

\clearpage

\section{Tabular example}
\label{sec:dynamics}

To build some intuition for IPO-MD, in Figure~\ref{fig:example} we provide a plot of the trajectories obtained with a variety of values of $\MixtureParam$ for the game with preference probabilities as displayed below, with $\tau=0.1$, and $\piref$ set to the uniform policy.

\begin{align*}
    \begin{pmatrix}
    0.5 & 0.8 & 0.1 \\
    0.1 & 0.5 & 0.8 \\
    0.9 & 0.1 & 0.5
    \end{pmatrix}
\end{align*}

\begin{figure}[h]
    \centering
    \includegraphics[width=.9\textwidth]{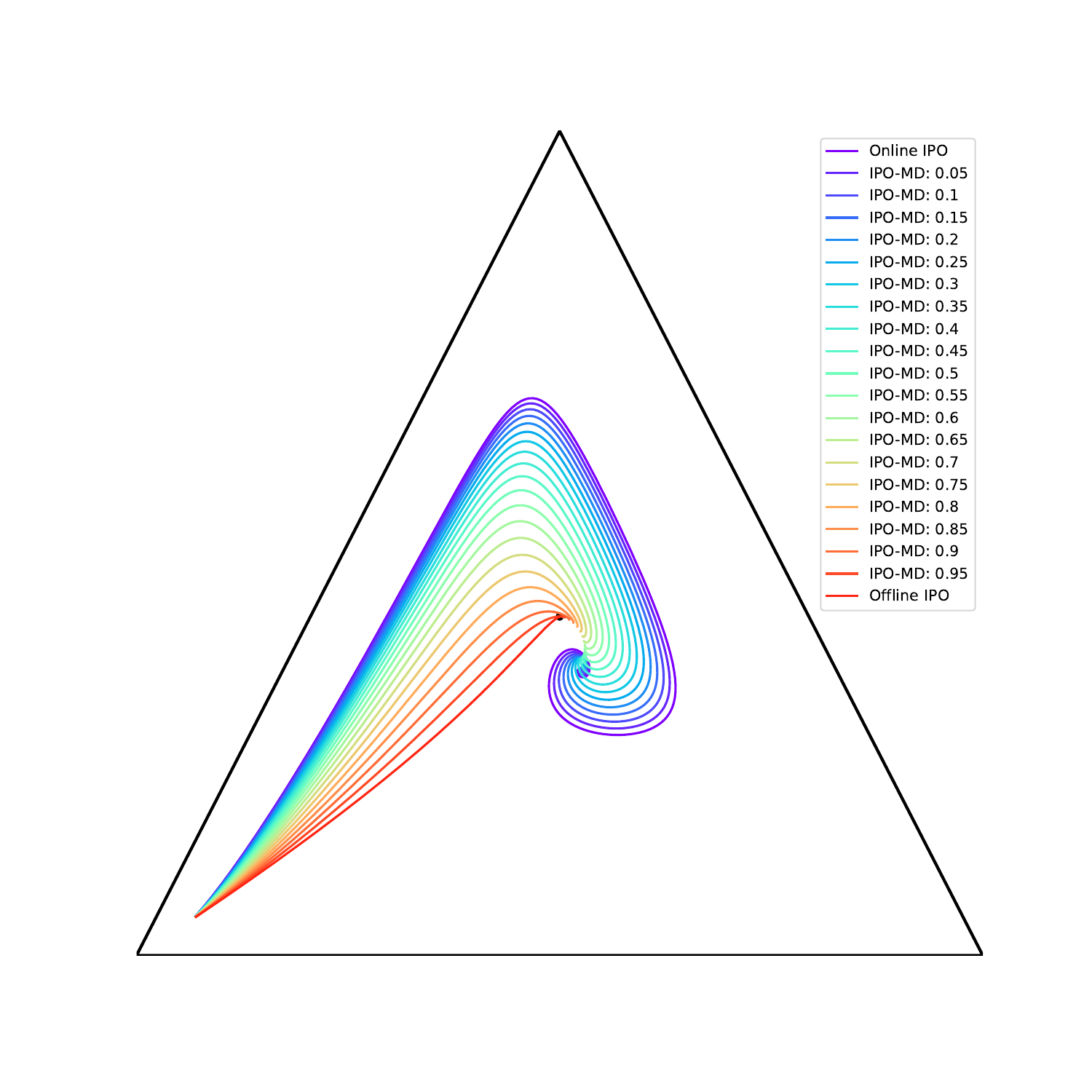}
    \caption{Online, offline, and MD variants of IPO.}
    \label{fig:example}
\end{figure}

\clearpage

\section{Supplementary Theoretical Study of Online DPO}
\label{sec:online-dpo-supp}

In this section, we explore whether online DPO is related to the regularised game given \cref{eq:payoff}, given that \cref{prop:online-ipo-nash} shows a similar relationship for online IPO and the regularised game.

Our strategy is to derive the gradient of the offline DPO objective (\cref{lem:offline-dpo-gradient}), and then inspect whether different sampling distributions $\mu$ and candidate solutions $\pi$ satisfy the KKT conditions \citep{boyd2004convex} for the minimising the objective.
We can also explore online DPO, by inspecting the KKT conditions when the sampling distribution matches the candidate solution ($\mu = \pi$).
For example, in \cref{lem:online-dpo-kkt} we present conditions for the solution of the regularised game (see \cref{eq:online-fixed-point}) to also be a stationary point of online DPO.
Given that the minimiser of the online online IPO objective is the Nash equilibrium of the regularised game given in \cref{eq:payoff} (\cref{prop:online-ipo-nash}), \cref{lem:online-dpo-kkt} gives us a condition for the online DPO problem to be equivalent (solution-wise) to the online IPO problem.
However, the condition seems difficult to satisfy: For example, there is no 2-action problem for which the condition is satisfied with the exception when preferences are uniform ($p(1 \succ 2) = \frac{1}{2}$).
In this sense, apart from the trivial uniform-preference case, online DPO and online IPO are different objectives when $|\mathcal{Y}| = 2$.

Regaring stationary points of online DPO, we show that, under the Bradley-Terry model assumption, the RLHF solution (\cref{eq:rlhf-closed-form}) is a stationary point of online DPO  (\cref{thm:online-dpo-kkt}).
We note in passing (\cref{remark:problematic-dpo-solutions}) that the RLHF solution is not the only solution of offline DPO when $\mu(y) = 0$ for some $y \in \mathcal{Y}$.
In fact, there are infinitely many, with arbitrarily small probabilities for $y$ such that $\mu(y) > 0$.
So while we can say that stationary points of online DPO are also stationary points of offline DPO, a solution for the offline DPO problem with a given $\mu$ may not be a solution for corresponding the online DPO problem, because of what happens for $y$ such that $\mu(y) = 0$.

\subsection{Results and Proofs}

We let $\Delta^\circ(\mathcal{Y}) \doteq \left\{ p \in \Delta(\mathcal{Y}) : p(y) > 0, \forall y \in \mathcal{Y} \right\}$ be the interior of the simplex.

\begin{proposition}
\label{prop:log-sigmoid-derivative}
For $t \in \mathbb{R}$, we have
\[
  \frac{d}{dt}\log\sigma(t) = \sigma(-t).
\]
\end{proposition}

\begin{proof}
We have:
\begin{align*}
    \frac{d}{dt}\log\sigma(t) 
    &= \frac{d}{dt} -\log(1 + e^{-t})
    = -\frac{1}{1 + e^{-t}}\cdot\frac{d}{dt}(1 + e^{-t}) \\
    &= \frac{e^{-t}}{1 + e^{-t}}
    = \frac{1}{1 + e^t}
    = \sigma(-t).
\end{align*}
\end{proof}

\begin{proposition}
\label{prop:sigmoid-sums-to-one}
For $t \in \mathbb{R}$, we have
\[
  \sigma(t) + \sigma(-t) = 1
\]
\end{proposition}

\begin{proof}
We have:
\begin{align*}
    \sigma(t) + \sigma(-t) 
    = \frac{1}{1 + e^{-t}} + \frac{1}{1 + e^t}
    = \frac{1}{1 + e^{t}} + \frac{e^{-t}}{1 + e^{-t}}
    = 1.
\end{align*}
\end{proof}

\begin{lemma}
\label{lem:offline-dpo-gradient}
    Assume that $|\mathcal{Y}| < \infty$ and that $p(y \succ y') = 1 - p(y' \succ y)$ for all $y, y' \in \mathcal{Y}$.
    The offline DPO problem with sampling distribution $\mu$ can be written as
    \[
        \sup_{\pi \in \Delta(\mathcal{Y})} J_{\mathrm{DPO}}(\pi),
    \]
    where 
    \begin{align*}
    J_{\mathrm{DPO}}(\pi) &\doteq \sum_{y, y'} \mu(y)\mu(y')\left( p(y \succ y')\log\sigma\left(\tau\log\frac{\pi(y)\piref(y')}{\pi(y')\piref(y)}\right)\right. \\
    &\phantom{\doteq}+ \left.(1 - p(y \succ y'))\log\sigma\left(-\tau\log\frac{\pi(y)\piref(y')}{\pi(y')\piref(y)}\right) \right).
    \end{align*}
    Moreover, for $\pi \in \Delta^\circ(\mathcal{Y})$,
    \begin{equation}
        \label{eq:offline-dpo-gradient}
        \nabla J_{\mathrm{DPO}}(\pi)_y =
        2\tau\frac{\mu(y)}{\pi(y)}\sum_{y'}\mu(y') \left( p(y \succ y') - \sigma\left(\tau\log\frac{\pi(y)\piref(y')}{\pi(y')\piref(y)}\right)\right).
    \end{equation}
\end{lemma}

\begin{proof}
We will use the shorthands $\mu_{y, y'} \equiv \mu(y)\mu(y')$, $\pi_y \equiv \pi(y)$ (and so forth) and $s_{y,y'} \doteq \tau \log\frac{\pi(y)\piref(y')}{\pi(y')\piref(y)}$.

The offline DPO objective is obtained by taking the expectation of the DPO loss \citep{rafailov2023direct} with $Y^+, Y^- \sim (\mu, \lambda_p)$:
\[
    \sum_{y, y'} \mu_{y, y'}( p(y \succ y')\log\sigma(s_{y,y'}) + (1 - p(y \succ y'))\log\sigma(-s_{y,y'})) = J_{\mathrm{DPO}}(\pi).
\]

Before looking at $\nabla J_{\mathrm{DPO}}(\pi)$, we will make some simplifying observations.
For $y'' \neq y, y'$,
\begin{equation}
\label{eq:zero-derivative}
    \frac{d}{d\pi_{y''}}s_{y, y'} = 0.
\end{equation}

Since $\mu_{y, y'} = \mu_{y', y}$, $s_{y,y'} = -s_{y', y}$ and $p(y \succ y') = 1 - p(y' \succ y)$, we have
\begin{equation}
\label{eq:symmetric-objective}    
\begin{aligned}
    &\mu_{y, y'}\left( p(y \succ y')\log\sigma(s_{y,y'}) + (1 - p(y \succ y'))\log\sigma(-s_{y,y'}) \right) \\
    &=\mu_{y', y} \left( (1 - p(y' \succ y))\log\sigma(-s_{y',y}) + p(y' \succ y)\log\sigma(s_{y', y}) \right).
\end{aligned}
\end{equation}

Also note that 
\begin{equation}
\label{eq:dsdf}
    \frac{d}{d\pi_y}s_{y,y'} = \frac{d}{d\pi_y}\left(\tau\log\frac{\pi_y\piref(y')}{\pi_{y'}\piref(y)}\right) = \tau\frac{d}{d\pi_y}\log\pi_y = \frac{\tau}{\pi_y}.
\end{equation}

Therefore, for all $y \in \mathcal{Y}$
\begin{align*}
    \frac{d}{d\pi_y}J_{\mathrm{DPO}}(\pi) 
    &= 2\sum_{y'}\mu_{y, y'}\left( p(y \succ y')\frac{d}{d\pi_y}\log\sigma(s_{y,y'}) + (1 - p(y \succ y'))\frac{d}{d\pi_y}\log\sigma(-s_{y,y'})\right), 
        &\mbox{(\cref{eq:zero-derivative,eq:symmetric-objective})} \\
    &= 2\sum_{y'}\mu_{y,y'}\left( p(y \succ y')\sigma(-s_{y,y'})\frac{d}{d\pi_y}s_{y,y'} - (1 - p(y \succ y'))\sigma(s_{y,y'})\frac{d}{d\pi_y}s_{y,y'}\right)
        &\mbox{(\cref{prop:log-sigmoid-derivative})} \\
    &= 2\sum_{y'}\mu_{y,y'}\left( p(y \succ y')\sigma(-s_{y,y'}) - (1 - p(y \succ y'))\sigma(s_{y,y'})\right)\frac{d}{d\pi_y}s_{y,y'} \\
    &= 2\sum_{y'}\mu_{y,y'} \left( p(y \succ y') - \sigma(s_{y,y'})\right)\frac{d}{d\pi_y}s_{y,y'}
        &\mbox{(\cref{prop:sigmoid-sums-to-one})} \\
    &= 2\sum_{y'}\mu_{y,y'} \left( p(y \succ y') - \sigma\left(\tau\log\frac{\pi_y\piref_{y'}}{\pi_{y'}\piref_y}\right)\right)\frac{\tau}{\pi_y},
        &\mbox{(\cref{eq:dsdf})} \\
    &= 2\tau\frac{\mu_y}{\pi_y}\sum_{y'}\mu_{y,y'} \left( p(y \succ y') - \sigma\left(\tau\log\frac{\pi_y\piref_{y'}}{\pi_{y'}\piref_y}\right)\right).
\end{align*}
\end{proof}

\begin{remark}
We have restricted the statement of \cref{lem:offline-dpo-gradient} to the interior of the simplex to stay consistent with the DPO formulation,
however \cref{eq:offline-dpo-gradient} holds for any function $\exp \circ f$ with $f: \mathcal{Y} \rightarrow \mathbb{R}$.
To see this, it suffices to notice that $J_{\mathrm{DPO}}$ is a function of the ratios $\frac{\pi_i}{\pi_j}$, so for any $\alpha > 0$ $J_{\mathrm{DPO}}(\pi) = J_{\mathrm{DPO}}(\alpha \cdot \pi)$ (where $\alpha \cdot \pi \doteq y \mapsto \alpha \cdot \pi(y)$).
\end{remark}

\begin{lemma}
\label{lem:online-dpo-kkt}
The Nash equilibrium $\pi^*$ of the regularised game given in \cref{eq:payoff} is a stationary point of online DPO iff:
\begin{equation}
    \label{eq:online-dpo-stationary-point-condition}
    p(y \succ \pi^*) = \sum_{y'}\pi^*(y')\sigma( p(y \succ \pi^*) - p(y' \succ \pi^*) ).
\end{equation}
\end{lemma}

\begin{proof}
Let $\pi^*$ be the Nash equilibrium of the regularised game given in \cref{eq:payoff}, the fixed point in \cref{eq:online-fixed-point}:
\begin{equation}
    \label{eq:pistar-fixed-point}
    \pi^*(y) \propto \piref(y) \exp\left( \frac{1}{\tau}p(y \succ \pi^*) \right).
\end{equation}

$\pi^*$ will be a stationary point of online DPO iff $\pi = \pi^*$ is a solution of the offline DPO problem with $\mu = \pi^*$.
Since $\pi^* \in \Delta^\circ(\mathcal{Y})$, we can use \cref{lem:offline-dpo-gradient}, $\pi^*$ is a solution of the offline DPO problem iff
\[
    \nabla J_{\mathrm{DPO}}(\pi^*) = 0.
\]

We have for all $y \in \mathcal{Y}$:
\begin{align*}
    \frac{d}{d\pi_y}J_{\mathrm{DPO}}(\pi^*) &= 2\tau\frac{\mu(y)}{\pi^*(y)}\sum_{y'}\mu(y') \left( p(y \succ y') - \sigma\left(\tau\log\frac{\pi^*(y)\piref(y')}{\pi^*(y')\piref(y)}\right)\right)
        &\mbox{(\cref{eq:offline-dpo-gradient})} \\
    &= 2\tau\sum_{y'}\pi^*(y') \left( p(y \succ y') - \sigma\left(\tau\log\frac{\pi^*(y)\piref(y')}{\pi^*(y')\piref(y)}\right)\right)
        &\mbox{($\mu = \pi^*$)} \\
    &= 2\tau\sum_{y'}\pi^*(y') \left( p(y \succ y') - \sigma\left( p(y \succ \pi^*) - p(y' \succ \pi^*) \right)\right)
        &\mbox{(\cref{eq:pistar-fixed-point})} \\
    &= 2\tau\left( p(y \succ \pi^*) -\sum_{y'}\pi^*(y')\sigma\left( p(y \succ \pi^*) - p(y' \succ \pi^*) \right)\right),
\end{align*}
and the result follows by using the fact that $\tau > 0$.
\end{proof}

\begin{theorem}
\label{thm:two-class-nash-is-not-dpo-solution}
No 2-action regularised game given in \cref{eq:payoff} has a Nash equilibrium that satisfies \cref{eq:online-dpo-stationary-point-condition}, except for the regularised games with $p(y_1 \succ y_2) = \frac{1}{2}$.
\end{theorem}

\begin{proof}
Let $\mathcal{Y} = \{1, 2\}$.
For this two-action problem, we can write the preference matrix as
\[
P \doteq \left(
\begin{array}{cc}
    \frac{1}{2} & 1 - p \\
    p & \frac{1}{2}
\end{array}
\right),
\]
where $P_{yy'} = p(y \succ y')$.

Let $\alpha$ be such that $\pi^* = (\alpha, 1 - \alpha)^\top$.
Then
\[
p(y \succ \pi^*) = P\pi^* = \left(1 - \frac{\alpha}{2} - p + \alpha p, \frac{1}{2} - \frac{\alpha}{2} + \alpha p\right)^\top,
\]
and
\[
    p(1 \succ \pi^*) - p(2 \succ \pi^*) = \frac{1}{2} - p.
\]

Then the difference between both sides of \cref{eq:online-dpo-stationary-point-condition} for $y = 1$ is:
\begin{align*}
    &p(1 \succ \pi^*) -  \sum_{y'}\pi^*(y')\sigma(p(1 \succ \pi^*) - p(y' \succ \pi^*)) \\
    &= 1 - \frac{\alpha}{2} - p + \alpha p - \alpha\sigma(0) - (1 - \alpha)\sigma\left(\frac{1}{2} - p\right) \\
    &=  1 - \alpha - p + \alpha p - (1 - \alpha)\sigma\left(\frac{1}{2} - p\right) \\
    &= (1 - \alpha)\left( 1 - p - \sigma\left(\frac{1}{2} - p\right)\right).
\end{align*}

For $y = 2$ we get
\begin{align*}
    &p(2 \succ \pi^*) - \sum_{y'}\pi^*(y')\sigma(p(2 \succ \pi^*) - p(y' \succ \pi^*)) \\
    &=\alpha p + \frac{(1 - \alpha)}{2} - \alpha\sigma\left(p - \frac{1}{2}\right) - (1 - \alpha)\sigma(0) \\
    &= \alpha p - \alpha\sigma\left(p - \frac{1}{2}\right) \\
    &= \alpha p - \alpha + \alpha\sigma\left(\frac{1}{2} - p\right) \\
    &= -\alpha\left( 1 - p - \sigma\left(\frac{1}{2} - p\right) \right).
\end{align*}

Now, if we let $\varepsilon \doteq \frac{1}{2} - p$,
we can see that for $p < \frac{1}{2}$
\[
    \sigma\left(\frac{1}{2} - p\right) = \sigma(\varepsilon) < \sigma(0) = \frac{1}{2} < \frac{1}{2} + \varepsilon = 1 - p,
\]
so (considering the analogous case for $p < \frac{1}{2}$)
\[
    1 - p - \sigma\left(\frac{1}{2} - p\right)
    \begin{cases}
    > 0, & p < \frac{1}{2}, \\
    = 0, & p = \frac{1}{2}, \\
    < 0, & p > \frac{1}{2}.
    \end{cases}
\]

Therefore, if $p \neq \frac{1}{2}$, we cannot satisfy \cref{eq:online-dpo-stationary-point-condition} for $y = 1$ and $y = 2$ (note that $\alpha = 0$ or $\alpha = 1$ satisfy the equation for only one $y$).
\end{proof}

\begin{theorem}
\label{thm:online-dpo-kkt}
Assume that the preferences admit a Bradley-Terry model \citep{rafailov2023direct}, that is, there exists $r: \mathcal{Y} \rightarrow \mathbb{R}$ such that for all $y, y' \in \mathcal{Y}$
\[
    p(y \succ y') = \sigma( r(y) - r(y') ),
\]
Then
\begin{equation}
    \label{eq:dpo-pi-r}
    \pi^r(y) \propto \piref(y) \exp\left( \frac{1}{\tau}r(y) \right)
\end{equation}
is a critical point for offline IPO for any $\mu$, and a stationary point of online DPO.
\end{theorem}

\begin{proof}
First, assume that $\piref \in \Delta^\circ(\mathcal{Y})$ (we will deal with the general case at the end).

Let us consider the offline DPO problem with sampling distribution $\mu$.
Assuming that the preferences admit a Bradley-Terry model, we know from \citet{rafailov2023direct} that the solution of the offline DPO problem is given by
\begin{equation}
    \label{eq:dpo-pi-r-2}
    \pi^r(y) \propto \piref(y) \exp\left( \frac{1}{\tau}r(y) \right).
\end{equation}

Indeed, we can verify that for any $\mu$
\begin{align*}
     \nabla J_{\mathrm{DPO}}(\pi^r) &=
        2\tau\frac{\mu(y)}{\pi(y)}\sum_{y'}\mu(y') \left( p(y \succ y') - \sigma\left(\tau\log\frac{\pi(y)\piref(y')}{\pi(y')\piref(y)}\right)\right)
            &\mbox{($\pi^r \in \Delta^\circ(\mathcal{Y})$, \cref{eq:offline-dpo-gradient})} \\
        &= 2\tau\frac{\mu(y)}{\pi^r(y)}\sum_{y'}\mu(y') \left( p(y \succ y') - \sigma\left(r(y) - r(y')\right)\right)
            &\mbox{(\cref{eq:dpo-pi-r})} \\
        &= 2\tau\frac{\mu(y)}{\pi(y)}\sum_{y'}\mu(y') \left( p(y \succ y') - p(y \succ y')\right)
            &\mbox{(Bradley-Terry model assumption)} \\
        &= 0
\end{align*}

It follows that $\pi^r$ is an offline DPO solution under any sampling distribution $\mu$, and if $\mu \in \Delta^\circ(\mathcal{Y})$ then $\pi^r$ is the only solution.
In particular this means that $\pi^r$ is also a stationary point for online DPO (by taking $\mu = \pi^r$).

In the case where $\piref \in \Delta(\mathcal{Y}) - \Delta^\circ(\mathcal{Y})$, we can prove the result as follows.
Let $\mathcal{Y}' \doteq \{y : \piref(y) = 0\}$.
For $y \in \mathcal{Y}'$, we have $\pi^r(y) = 0$ by \cref{eq:dpo-pi-r}.
Since $\pi^r$ is the sampling distribution, the gradient of the offline DPO objective in $\mathcal{Y'}$ is zero (and it does not matter that $\log \pi^r(y)$ and $\log \piref(y)$ are undefined since they do not appear in the objective).
Now it suffices to note that $\pi^r$ restricted to $\mathcal{Y} - \mathcal{Y}'$ is a distribution in $\Delta^\circ(\mathcal{Y} - \mathcal{Y}')$, so we can use the previous case ($\pi^r \in \Delta^\circ(\mathcal{Y}))$ to show the result, and it follows that $\pi^r$ satisfies the KKT conditions for all $y \in \mathcal{Y}$.
\end{proof}

\begin{remark}
In the special case of rock-paper-scissors preferences (which do not admit a Bradley-Terry model) with $\piref$ uniform and $\tau > 0$, we can still satisfy \cref{eq:online-dpo-stationary-point-condition}, since $\pi^*$ is also uniform and $p(y \succ \pi^*) = \frac{1}{2}$.
\end{remark}

\begin{remark}
\label{remark:problematic-dpo-solutions}
In offline DPO, the presence of sets of measure zero guarantees the solution is not unique.
Let $\pi^{\mathrm{DPO}}$ be a solution of offline DPO under sampling distribution $\mu$ (assume it exists). 
$\pi^{\mathrm{DPO}}$ may not necessarily be $\pi^r$ from \eqref{eq:dpo-pi-r} (even if we assume that the preferences admit a Bradley-Terry model).

To see this, assume that there exists $\mathcal{Y}' \subset \mathcal{Y}$ such that $\mathcal{Y}' \neq \emptyset$ and $\mu(y) = 0$ for all $y \in \mathcal{Y}'$.
For any $\alpha \in (0, 1]$, define $\pi^\alpha \in \Delta^\circ(\mathcal{Y})$ by
\[
    \pi^\alpha(y) \doteq
    \begin{cases}
    \alpha \cdot \pi^{\mathrm{DPO}}(y), & y \notin \mathcal{Y}', \\
    \mbox{arbitrary}, & y \in \mathcal{Y}'.
    \end{cases}
\]
By \cref{lem:offline-dpo-gradient}, $\nabla J_{\mathrm{DPO}}(\pi^\alpha) = 0$, so $\pi^\alpha$ is a minimum of the offline DPO objective.
In fact, this holds for any $\alpha > 0$, no matter how small, and this also means that if $\mu \notin \Delta^\circ(\mathcal{Y})$ we can find offline DPO solutions with arbitrarily small probabilities for sampled $y$ ($y$ such that $\mu(y) > 0$).
\end{remark}

\end{document}